\documentclass[twoside]{article}

\usepackage[accepted]{aistats2025}

\usepackage{microtype}
\usepackage{graphicx}
\usepackage{booktabs} 
\usepackage{hyperref}

 \usepackage{enumitem}

\usepackage{amssymb}
\usepackage{mathtools}
\usepackage{amsthm}
\usepackage{bbm}
\usepackage[capitalize,noabbrev]{cleveref}

\usepackage{parskip}
\usepackage[utf8]{inputenc} 
\usepackage[T1]{fontenc}    
\usepackage{microtype}
\usepackage{graphicx}
\usepackage{booktabs} 
\usepackage{pgfplots}
\usepgfplotslibrary{groupplots}
\usepgfplotslibrary{statistics}
\usepackage{pifont}

\usepackage{hyperref}

\usepackage{ulem}

\usepackage{multicol}
\usepackage{etoolbox}
\usepackage{xcolor}
\definecolor{orange}{rgb}{0.2,0.7,0.4}
\usepackage{layout}
\usepackage{wrapfig}

\graphicspath{{figures/}}
\usepackage[small]{caption}
\usepackage{subcaption}
\usepackage{amsfonts}
\usepackage{mwe}
\usepackage{multirow}

\usepackage{stmaryrd}
\usepackage{enumitem}
\usepackage{makecell}
\usepackage{tabularx}
\usepackage{float}

\newcommand{\ourmod}{{EPIC}}
\input{math_commands.tex}
\usepackage{xcolor}

\usepackage{colortbl}
\definecolor{pearThree}{HTML}{E74C3C}
\definecolor{pearcomp}{HTML}{B97E29}
\definecolor{pearDark}{HTML}{2980B9}
\definecolor{pearDarker}{HTML}{1D2DEC}
\hypersetup{
	colorlinks,
	citecolor=pearDark,
	linkcolor=pearThree,
	breaklinks=true,
	urlcolor=pearDarker}

\definecolor{HighlightColor}{gray}{0.97}
\definecolor{aliceblue}{rgb}{0.94, 0.97, 1.0}
\definecolor{palecornflowerblue}{rgb}{0.67, 0.8, 0.94}
\definecolor{paleaqua}{rgb}{0.74, 0.83, 0.9}

\definecolor{linen}{rgb}{0.98, 0.94, 0.9}
\definecolor{magnolia}{rgb}{0.97, 0.96, 1.0}
\definecolor{mistyrose}{rgb}{1.0, 0.89, 0.88}
\definecolor{piggypink}{rgb}{0.99, 0.87, 0.9}
\colorlet{colorast}{red!80!black}

\usepackage[round]{natbib}

\runningtitle{Statistical Guarantees for Lifelong Reinforcement Learning}
\runningauthor{Zhang, Zhi et al. }

\begin{document}

\twocolumn[

\aistatstitle{Statistical Guarantees for Lifelong Reinforcement Learning \\ using PAC-Bayes Theory}
\aistatsauthor{ Zhi Zhang$^{1}$ \And  Chris Chow$^{2}$ \And Yasi Zhang$^{1}$  \And  Yanchao Sun$^{3}$    }

\aistatsauthor{Haochen Zhang$^{1}$  \And Eric Hanchen Jiang$^{1}$ \And Han Liu$^{4}$ \And  Furong Huang$^{5}$}

\aistatsauthor{Yuchen Cui$^{1}$\And Oscar Hernan Madrid Padilla$^{1}$ }

\vspace{0.2cm} 
\begin{center}
    {\small 
    $^{1}$University of California, Los Angeles \quad
    $^{2}$Niantic Labs \quad
    $^{3}$Apple Inc. \quad\\
    $^{4}$Northwestern University \quad
    $^{5}$University of Maryland, College Park
    }
\end{center}
\vspace{0.4cm} 
]

\begin{abstract}
Lifelong reinforcement learning (RL) has been developed as a paradigm for extending single-task RL to more realistic, dynamic settings. In lifelong RL, the "life" of an RL agent is modeled as a stream of tasks drawn from a task distribution. We propose EPIC (\underline{E}mpirical \underline{P}AC-Bayes that \underline{I}mproves \underline{C}ontinuously), a novel algorithm designed for lifelong RL using PAC-Bayes theory. EPIC learns a shared policy distribution, referred to as the \textit{world policy}, which enables rapid adaptation to new tasks while retaining valuable knowledge from previous experiences. Our theoretical analysis establishes a relationship between the algorithm's generalization performance and the number of prior tasks preserved in memory. We also derive the sample complexity of EPIC in terms of RL regret. Extensive experiments on a variety of environments demonstrate that EPIC significantly outperforms existing methods in lifelong RL, offering both theoretical guarantees and practical efficacy through the use of the world policy.
\end{abstract}

\section{INTRODUCTION}

Deep reinforcement learning (RL) has excelled in challenging tasks including abstract strategy games \citep{silver2017mastering, silver2016mastering}, visual navigation \citep{zhu2017target}, and control \citep{mnih2015human, lillicrap2015continuous}. 
However, RL is a data intensive learning paradigm, therefore training a policy for every task from scratch is computationally expensive and time-consuming. 
In reality, many tasks an agent encounters are not entirely novel but instead belong to a broader task distribution, meaning they share commonalities that can be leveraged. 
This insight highlights the inefficiency of retraining for every individual task. Lifelong RL, also known as continual RL, emerges as a promising framework where an agent interacts with a sequence of tasks, continuously adapting and improving its policy by leveraging knowledge from past task instances \citep{khetarpal2022towards}. 

In lifelong RL, an agent's objectives are primarily to achieve \textbf{fast adaptation} with limited samples and effective \textbf{knowledge retention} \citep{abel2024definition}. In other words, lifelong RL agents experience a stability-plasticity dilemma, where the agent must balance retaining useful long-term knowledge with the ability to rapidly adapt to new situations.  
 
Recent  methods addressing knowledge retention and transfer in lifelong RL include
Q-value function transfer \citep{lecarpentier2021lipschitz}, optimizing Q-value function initialization \citep{abel2018policy},  decomposing the value function into permanent and transient components \citep{anand2023prediction}, 
reusing knowledge by sampling from past experiences \citep{kessler2023effectiveness}, detecting change points in rewards and environment dynamics \citep{steinparz2022reactive}, and using a Bayesian approach to learn a common task distribution for better data efficiency and transfer \citep{fu2022model}.

In lifelong RL, domain shifts induce non-stationarity, which occurs not only due to changing transition dynamics and reward functions, but also through variations in available actions or decisions over time \citep{boutilier2018planning,chandak2020lifelong}. Such scenarios are common in real-world applications. For example, in robotics, additional control components are integrated throughout the robot's lifetime, and in medical decision support systems, new treatments and medications must be incorporated.

Furthermore, tasks encountered over an agent's lifetime can be highly diverse, yet certain high-level strategies that can be shared across tasks. Not only should the world model \citep{ha2018world,fu2022model,anand2023prediction}, which captures the general knowledge distribution of tasks, be continuously refined, but it is also crucial for an agent to develop a \textit{world policy}.
The key consideration of this \textit{world policy} is to adapt the parameters of the policy so they are captured by a global distribution, which represents the uncertainty over policy parameters. This allows for better generalization and adaptability across tasks.

Motivated by the above need,
we raise two questions:
\begin{enumerate}
\item Can we extract the common structure present in policies from previously encountered tasks, allowing the agent to quickly learn the policy specific to new task to enable fast adaptation \textit{with} theoretical guarantees? 
    \item How many samples are required to achieve a given level of performance?
\end{enumerate}

To answer these two questions, we develop a unified framework based on a Bayesian method that learns a rapidly adaptable policy distribution 
from past tasks, retaining valuable information while remaining capable of quickly adapting to unseen situations. 

We also provide a theoretical analysis of the algorithm's generalization performance relative to the number of effective tasks and retained knowledge in the finite Markov Decision Process (MDP) setting, along with its sample complexity to demonstrate efficiency.

When addressing the first question, we must also account for both catastrophic forgetting and
generalizability -- the aforementioned \textit{stability-plasticity dilemma}. Agents that can quickly solve traditional RL problems risk abruptly losing performance on tasks they have seen before due to their flexibility or plasticity. On the other hand, agents that do not forget any of their past experience may give up a measure of their plasticity. 
These issues are central in lifelong RL, and can be approached from a Bayesian perspective \citep{khetarpal2022towards}. Bayesian methods have been applied to meta learning \citep{amit2018meta}, lifelong learning for bandits \citep{flynn2022pac}, and learning controls for robots in multiple environments \citep{majumdar2021pac}, aiming to learn a fast adapted policy.
Instead of learning a specific policy, we leverage the PAC-Bayes 
theory to learn a distribution of policy hypotheses shared across multiple tasks.  Further details about PAC-Bayes theory can be found in Section~\ref{sec:pacbayes} and the Related Works section in Appendix \S\ref{sec_related_works}. When a new task arises, we can initialize a policy hypothesis by sampling from this learned distribution. A well-constructed distribution of hypotheses promotes effective long-term memory, mitigating catastrophic forgetting. Unlike prior methods, we sample a random policy function according to this distribution. This function sampling approach is seen in modern popular deep learning methods, for example, in-context learning \citep{garg2022can}.

For the second question, an agent has to keep learning as well as forgetting. Too much experienced knowledge kept in memory may decrease the learning efficiency, while too little may be insufficient to learn an effective policy distribution.  To address the second question, we derive a relationship between the performance of our algorithm and the number of experienced tasks (denoted as $N$ in a later section) that need to be retained in the agent's memory, based on PAC-Bayes theory. We use the negative expected long term rewards, where the expectation is taken with respect to tasks and policies (also known as the generalization error), as a measure of the algorithm's performance from a statistical perspective.    

From our theoretical result, where we provide an expression of this relationship, we discovered a trade-off between this value and the algorithm's performance, which aligns with natural intuition, a double sided effect. In practice, our expression allows us to optimize the performance of our algorithm by optimizing $N$, although we recommend using hyperparameter tuning.  
Furthermore, to demonstrate the efficiency of our algorithm, we derive its sample complexity from a RL regret perspective, showing that our algorithm learns an optimal policy as more tasks encountered.

\noindent \textbf{Our Contributions.}
In this work, we introduce a novel PAC-Bayes framework tailored to lifelong RL, addressing critical challenges including changing decisions, catastrophic forgetting and efficient knowledge retention. Our contributions are summarized as follows:\vspace{-5pt}
\begin{itemize}
\item We propose EPIC (\underline{E}mpirical \underline{P}AC-Bayes that \underline{I}mproves \underline{C}ontinuously), a lifelong RL algorithm that leverages PAC-Bayes theory to learn a shared policy distribution, referred to as the \textit{world policy}. This world policy enables the agent to quickly adapt to new tasks while retaining useful knowledge from past experiences, providing theoretical guarantees of generalization across tasks.

    \vspace{-3pt}
   \item We derive a novel PAC-Bayes bound for lifelong RL and provide a theoretical analysis that links long-term rewards to the number of retained past tasks, ensuring a balance between memory usage and performance across diverse tasks. We provide a sample complexity of our approach in terms of RL regret. 
    \vspace{-3pt}
    \item 
   We evaluate EPIC through extensive numerical experiments with common lifelong RL benchmarks, as well as additional environments we created.  Our results show EPIC outperforms prior methods.
These results underscore EPIC's effectiveness in lifelong learning scenarios, offering a robust and theoretically grounded solution for continual adaptation in RL. 
    
\end{itemize}

\section{RELATED WORKS}
\label{sec_related_works}

\paragraph{Lifelong Reinforcement Learning:} 
Lifelong learning has been a crucial area of research in machine learning, where the goal is to develop agents that can continuously adapt to new tasks while retaining knowledge from previous experiences. Early foundational works, such as \cite{naikMetaneuralNetworksThat1992} and \cite{thrunLearningLearnIntroduction1998}, explored the basic principles of lifelong learning, setting the stage for more advanced methods. Subsequent research has focused on mitigating catastrophic forgetting and enhancing data efficiency, which are critical challenges in lifelong learning scenarios.
Various approaches have been proposed to improve adaptation in lifelong learning. \cite{saxeExactSolutionsNonlinear2014,kirkpatrickOvercomingCatastrophicForgetting2017,krahenbuhlDatadependentInitializationsConvolutional2016,salimansWeightNormalizationSimple2016} explored strategies for better initialization in deep networks.

Lifelong RL, as an extension of lifelong learning, naturally aligns with the agent-environment interaction framework, making it ideal for continual learning \citep{khetarpal2022towards}. Prior works \citep{lecarpentier2021lipschitz, abel2018policy} emphasize value transfer and initialization to boost learning efficiency,  while \cite{chandak2020lifelong} tackles the challenge of evolving action sets. \cite{anand2023prediction} introduces a dual-component value function approach for balancing long-term stability and short-term adaptability, and \cite{fu2022model} develops a model-based Bayesian framework that enhances both forward and backward transfer by extracting common structures across tasks. Lifelong RL has been further formalized as a framework where agents continuously learn and adapt, moving beyond static solutions \citep{abel2024definition}.

Recent baseline algorithms for lifelong RL have made significant advancements. Continual Dreamer \citep{kessler2023effectiveness} employs ensemble networks and is task-agnostic, leveraging a world model that can generate tasks for improving learning efficiency. VBLRL \citep{fu2022model} is a model-based method that learns a Bayesian posterior distribution shared across tasks to increase sample efficiency in related tasks. LPG-FTW \citep{mendez2020lifelong} is a policy-gradient-based lifelong method that uses data from previously seen tasks to train policy networks, accelerating the learning of new tasks. EWC \citep{kirkpatrick2017overcoming} is a single-model lifelong RL algorithm that avoids forgetting by imposing a quadratic penalty, pulling weights back towards values important for previously learned tasks. T-HiP-MDP \citep{killian2017robust} is a model-based method that models related tasks using low-dimensional latent embeddings and a Bayesian Neural Network, which captures both shared dynamics across tasks and individual task variations.

Our approach introduces a lifelong RL framework integrating PAC-Bayes theory to learn a policy distribution in non-stationary environments, ensuring effective knowledge retention and adaptability across tasks throughout the agent's lifetime.

\textbf{PAC-Bayes Theory:} 
PAC-Bayes theory \citep{mcallesterPACBayesianTheorems1999} has been extensively used in supervised and deep learning to study generalization bounds \citep{langfordPACBayesMargins2002,seegerPACBayesianGeneralisationError2002,germainPACBayesianLearningLinear2009,dziugaite2020role,neyshaburPACBayesianApproachSpectrallyNormalized2018,neyshaburExploringGeneralizationDeep2017}. In recent years, PAC-Bayes theory has been applied to reinforcement learning (RL) \citep{schulman2015trust,fardPACBayesianModelSelection2010,fardPACBayesianPolicyEvaluation2012,majumdar2021pac,veerProbablyApproximatelyCorrect2020}, primarily focused on single-task or offline settings, providing a framework for deriving generalization bounds in dynamic and uncertain environments.  Our method uniquely integrates PAC-Bayes theory into lifelong RL, providing a framework for continuous learning and adaptation. \cite{mbacke2023statistical} is a recent seminal work, both their and our methods make a similar contribution to provide statistical guarantees for particular machine learning methods by using PAC-Bayes theory.

\section{PRELIMINARIES}
\subsection{Reinforcement Learning}
In RL, an agent interacts with the environment by taking actions, observing states and receiving rewards. The environment is modeled by a Markov Decision Process (MDP), which is denoted by a tuple $\mathcal{M}=\langle \mathcal{S}, \mathcal{A}, T, R, \gamma, \nu \rangle$, where $\mathcal{S}$ is the state space, $\mathcal{A}$ is the action space, $T$ is the transition kernel, $R$ is the reward function, $\gamma \in (0, 1)$ is the discount factor, and $\nu$ is the initial state distribution.

A trajectory $\tau \sim \pi$ generated by policy $\pi$ is a sequence $s_1, a_1, r_1, s_2, a_2, \cdots$, where $s_1 \sim \nu$, $a_t \sim \pi(a|s_t)$, $s_{t+1} \sim T(s| s_t, a_t)$ and $r_t = R(s_t, a_t)$.
The goal of an RL agent is to find an optimal policy $\pi^*$ that maximizes the expected total rewards $J(\pi) = \mathbb{E}_{\tau \sim \pi}[r(\tau)] = \mathbb{E}_{s_1,a_1,\cdots \sim \nu, \pi, T, R}[\sum_{t=1}^{\infty} \gamma^{t-1} r_t]$.
 
\subsection{Lifelong Reinforcement Learning}

In lifelong RL, the agent interacts with a (potentially infinite) sequence of tasks, which come from an underlying task distribution \citep{khetarpal2022towards}, denoted as $\mathcal{D}_i$, $i = 1, \ldots, \infty$. Suppose that tasks share the same $\gamma$, but may have different $\mathcal{S}$,$\mathcal{A}$, transition probabilities $T$ and rewards $R$. The learning process is:
\begin{enumerate}[topsep=0pt, itemsep=-1ex]
    \item Initialize a policy $\pi_0$;
    \item Sample a task (MDP) $\mathcal{M}_i \sim \mathcal{D}_i$;
    \item Starting from $\pi_0$, learn a policy $\pi_i$ for task $\gM_i$ to maximize rewards.
\end{enumerate}

An effective lifelong RL agent should quickly adapt to new tasks that it encounters throughout its life.
 
\subsection{PAC-Bayes Theory}
\label{sec:pacbayes}

PAC-Bayes analysis applies to learning algorithms that output a distribution over hypotheses $h\in \gH$. This refers to $h$ is sampled independently from a distribution over functions in a function class $\gH$. For example, for a linear predictor of $d$ dimension, $h(x) = \big\langle w , x \big\rangle$, we let $w \sim \gN (0, I_d)$. Generally, such algorithms will be given a prior distribution $\barb{P} \in \mathcal{P}$ at the beginning and learn a posterior distribution $P \in \mathcal{P}$ after observing training data samples $\{z_i \}_{i=1}^N$. We define the expected loss (generalization error) $l_{\gD}(P) = \Ept{h\sim P}{\Ept{z\sim\gD}{h(z)}}$, and the empirical loss (training error) $l_{\gS}(P) = \Ept{h\sim P}{\frac{1}{N}
\sum_{i=1}^N{h(z_i)}}$, which are under the expectation of hypothesis $h \sim P$. 

The main application of PAC-Bayes analysis in machine learning is to produce high-confidence bounds for the true or generalization error in terms of the training  error plus $\mathscr R(\mathbb{D}_{KL}(P \| \barb{P}))$, which is a function of the KL divergence between the prior and posterior distributions, as shown below \citep{mcallesterPACBayesianTheorems1999},

\begin{equation}
\begin{aligned}
\label{eqa:PAC-Upper-SL}
& l_{\mathcal{D}}(P) \leq U(P)\coloneqq l_{S}(P) + \mathscr R(\mathbb{D}_{KL}(P \| \barb{P})), 
\end{aligned}
\end{equation}
with
\begin{equation}
\begin{aligned}
 & \mathscr R(\mathbb{D}_{KL}(P \| \barb{P})) \\ &\coloneqq \sqrt{\frac{1}{2N} \left[\mathbb{D}_{KL}\left(P \| \barb{P}\right) + \log\left(2N^{1/2} / \delta\right)\right]},
\end{aligned}
\end{equation}
where $U(P)$ in right-hand side of Equation~\eqref{eqa:PAC-Upper-SL} is called the generalization error bound that depends on $P$,  
and minimization of this bound leads to generalization error guarantees.

\section{METHODS}
We propose a PAC-Bayes lifelong RL algorithm, EPIC (Algorithm~\ref{alg:mrl-0}), to minimize the novel bound in \eqref{eq2}. The algorithm utilizes a Bayesian posterior to distill the common policy distribution learned from previous tasks, which is then used to sample the policy and serves as a prior for new tasks. We provide a generalization guarantee for EPIC in \Cref{thm:main}. Furthermore, we employ the Gaussian family for the posterior and prior in EPICG (Algorithm~\ref{alg:mrl}). The sample complexity of EPICG is given in \Cref{lemma:sample complexity}.

\subsection{PAC-Bayes Framework for Lifelong RL}

We learn a general policy distribution $P$ for lifelong RL by leveraging the core concept of the PAC-Bayes Method. We explicitly formulate $U(P)$ for the lifelong RL setting and employ it to propose an algorithm that learn the $P$ by minimizing $U(P)$ to 
accomplish the lifelong learning objective.

Define $\mathcal{P}$ as the whole policy space for $P$.
Rather than considering a general distribution $P$ for hypotheses where $\Pi$ can be infinite, we let $\mathcal{P}$ be parameterized by $\theta \in \mathbb{R}^d$ such that $\theta \sim P$. Note $\theta$ could be a neural network.

Naturally, the distribution $P$ is the posterior distribution of policy $\theta$ in the PAC-Bayes framework. Then let $\barb{P}$ be the prior distribution of the parameter. In the lifelong setting, as the tasks stream in, assume the agent has encountered $K$ tasks so far,
then the PAC-Bayes lifelong RL problem is formulated as follows:
\beq
\label{eq2}
\min_{P} &  ~ U(P) \\ & \coloneqq  \frac{1}{K}\sum_{i=1}^K  \left\{
\mathop{\mathbb{E}}_{\theta\sim P} [- J_{\mathcal{M}_i}(\pi_{\theta}) ]\right\} + \mathscr R(\mathbb{D}_{KL}(P \| \barb{P})) , 
\beq 
where $\mathscr R(\mathbb{D}_{KL}(P \| \barb{P}))$ is derived later in our theory, 

where $J_{\mathcal{M}}(\pi_{\theta})$ is the total expected reward of policy $\pi$ in MDP $\mathcal{M}$, taking the expectation with respect to the posterior distribution $P$ for the parameter $\theta$. The negative sign can be interpreted as the loss on a specific task $\mathcal{M}$.

To be concrete, in the finite MDP setting with length $H$, for policy $\pi_{\theta}$ with $\theta \sim P(\theta)$, the total expected reward with task $\gM$ is the value function,    
\begin{align*}
J_{\mathcal{M}}(\pi_{\theta})   = \ept{\sum_{h=1}^{H-1}\gamma^{h-1}r_{h}|\pi_{\theta}, s_{1}, \gM},     
\end{align*}
from a length of $H$ consecutive sample transitions,  $s_{1}, a_{1}, r_{1}, s_{2}, a_{2}, r_{2}, \dots, s_{H} \sim \pi_{\theta} \times \gM$.

\subsection{An Algorithm based on PAC-Bayes Lifelong Framework}
We now develop an algorithm to exploit the PAC-Bayes framework to efficiently perform lifelong RL.

Consider a time where we have seen $K$ tasks so far, and denote them $\set{\mathcal{M}_i}_{i=1}^K$. They are drawn from the lifelong task distribution $\set{\mathcal{D}_i}_{i=1}^K$. 
Each distribution $\mathcal{D}_i$ should possess non-zero support and boundedness both from above and below. 
Critically, once the agent interacts with a task, revisiting previously encountered MDPs is not guaranteed.

Our objective is to learn a shared lifelong learning model - the distribution of $\theta$
using the $K$ tasks the algorithm has encountered so far.

To achieve this, we propose the following lifelong RL learning algorithm based on the learning objective in Equation \eqref{eq2}, and provide its theoretical justification. 
The main idea is to learn a policy distribution $P$ as a policy initializer using the objective in Equation \eqref{eq2}, referred to as the \textit{default policy} 
This approach allows the default policy to capture common knowledge among tasks, addressing the challenge of task divergence.

In the lifelong setting, the agent receives a new task, stores it, learns from it, then forgets. We allow the agent to keep a number of $N$ tasks in memory. We update the default policy every $N$ tasks and estimate the training cost based on the most recent $N$ tasks. 
At the $K$-th task, the agent has performed $\left\lfloor \frac{K}{N} \right\rfloor$ updates to the default policy so far. 
At each time step $l = 1, \cdots, \left\lfloor \frac{K}{N} \right\rfloor$, the agent has $\theta_{l-1}$ as its policy parameters from $P_{l-1}$. It encounters the $i$th task $\mathcal{M}_{l,i}$'s MDP, and receives  $J_{\mathcal{M}_{l,i}}(\pi_{\theta_{l-1}})$ as the total discounted expected reward. The collects trajectory data of $H$ steps for task $\mathcal{M}_{l,i}$, using $\pi_{\theta_{l-1}}$, resulting in a dataset $\tau_l = (\tau_{l,1}, \dots, \tau_{l,N})$ with a size of $\abs{\tau_l} = HN$.

The agent uses $\tau_l$ to update the default policy $P_l$ by minimizing the generalization error bound in \eqref{eq2}, evaluated at the current time's posterior $P_{l-1}$ and prior $\barb{P}_{l-1}$. Before learning starts, the agent initializes a prior policy distribution $\barb{P}_0$ and the same posterior policy distribution $P_0$ randomly or based on domain knowledge (Lines 2-3 of Algorithm~\ref{alg:mrl-0}).
Choosing a good prior policy distribution $\barb{P}_0$ is challenging as it affects the tightness of the bound.

\begin{algorithm}[t]
   \caption{Empirical PAC-Bayes that Improves Continuously (\ourmod)}
   \label{alg:mrl-0}
\begin{algorithmic}[1]
   \STATE {\bfseries Input:} Update frequency $N$;
    the number of steps allowed in each task $H$; prior evolving speed $\lambda$
   \STATE Initialize prior policy distribution $\barb{P}_0$
   \STATE Initialize default policy distribution $P_0 \gets \barb{P}_0$  
\FOR{$i=1,2,3,\cdots,K, \cdots, \infty $}
 \STATE Receive a new task $\mathcal{M}_i \sim \mathcal{D}_i$ and store it into Memory buffer
   \IF{$i \bmod N = 0$}
   \STATE Let $l = i/N$
        \STATE  Sample $\theta_{l-1} \sim P_{l-1}$ 
    \STATE{Roll out trajectories $\tau_{l,k}$ using $\pi_{\theta_{l-1}}$ and $\set{\gM_k}_{k=i-N+1}^i$ } and store $\tau_l$ into Memory.  
        \STATE {\# Update default policy $P_l$ by using $\tau_l$}
        \STATE $\barb{P}_{l-1} \gets (1-\lambda) \barb{P}_{l-1} + \lambda P_{l-1}$
        \STATE {{$P_l \gets \argmin_P U(P)$}}
        \STATE Decay $\lambda$ by $\lambda = \lambda \times \alpha$
        \STATE {Empty Memory by clearing dataset $\tau_{l}$}
   \ENDIF
   \ENDFOR
\end{algorithmic}
\end{algorithm}

  We adopt a Bayesian sequential experiment design \citep{chaloner1995bayesian} and use an evolving prior instead of a fixed one. We gradually move the prior towards the default policy by $\barb{P}_l = (1 - \lambda) P_l + \lambda \times \barb{P}_l$ (Line 11), where $\lambda \in (0,1)$ controls the moving speed, and $\lambda$ decays by $\lambda = \lambda \times \alpha$ ($\alpha < 1$) over the tasks. This allows us to find a good prior during learning and leverage it to improve the default policy. \label{sec:alpha}

As the agent encounters an increasing number of tasks, each task remains distinct. However, with more exposure to tasks, the agent gradually improves its understanding of the distribution $P$ for $\pi_{\theta}$. When a new task emerges, the agent can sample $\theta_{l} \sim P_{l-1}$, employing $\theta_l$ to generate a trajectory for subsequent updates. This allows the agent to learn faster, obtaining higher rewards in a shorter time frame. Next, we derive our main PAC-Bayes theorem for Algorithm \ref{alg:mrl-0}. 

Our learning process involves a loop of times to evolve the policy distribution. So we index the policy distribution at each time by a subscript. 
First, denote $\theta \coloneqq \left\{\theta_l \right\}_{l=0}^{\lfloor{\frac{K}{N}} \rfloor-1}$ and let $P\coloneqq P(\left\{\theta_l \right\}_{l=0}^{\lfloor{\frac{K}{N}} \rfloor-1})$ denote the joint posterior distribution of $\theta_0, \dots, \theta_{\lfloor{\frac{K}{N}} \rfloor-1}$ across all times. And naturally, let $P_l \coloneqq P(\theta_l|\theta_{l-1})$ be the conditional probability of policy for time $l$ given the policy from time $l-1$, and specially, let $P_0 \coloneqq P(\theta_0)$.

\begin{assumption}
\label{assump:1}
\textbf{(Conditional Independence)} Given the previous policy $\theta_{l-1}$, the current policy is conditionally independent of all earlier policies:
\begin{equation}
\theta_{l} \independent \theta_{l-2}, \dots, \theta_{0} \mid \theta_{l-1}.
\end{equation}
\textbf{(Policy Support Bound)} Define the smallest nonzero probability across all policies as:
\begin{equation}
s_{\min} = \inf \left\{ \min_{A: P_l(A) > 0} P_l(A) : P_l \in \Pi \right\}.
\end{equation}
\textbf{(Radius of Variation)} The maximum difference between consecutive policy distributions is bounded by:
\begin{equation}
r = \inf \left\{ c : \sup_{A \in \mathscr{A}} |P_l(A) - P_{l-1}(A)| \leq c, \quad \forall l \right\},
\end{equation}
where $\sup_{A \in \mathscr{A}} |P_l(A) - P_{l-1}(A)|$ is the total variation distance of $P_l$ and $P_{l-1}$ such that it measures the worst-case difference over all measurable events.
\end{assumption}



To simplify the analysis and improve readability, we denote $T = \left\lfloor \frac{K}{N} \right\rfloor$ and assume $K \bmod N = 0$ without loss of generality.
 Based on Assumption~\ref{assump:1}, we arrive at the following relationship:

\beq 
\label{eq7}
& P(\left\{\theta_l \right\}_{l=0}^{T-1}) 
= & P_{T-1}\times \dots \times P_l \times \dots \times P_{0}.
\beq 
 We also derive a corollary on the decomposition of the training error, which facilitates the subsequent proof. The proof is deferred to Appendix~\S\ref{sec_proof1}.
\begin{proposition}[Decomposition of Training Error]\label{cor_docomp}
    Suppose Assumption~\ref{assump:1} holds. Then we have:
\begin{align*}
   &\frac{1}{K} \sum_{i=1}^{K} 
\mathbb{E}_{\left\{\theta_l \right\}_{l=0}^{T-1} \sim P} [- J_{\mathcal{M}_i}(\pi_{\theta}) ] \\ & =\sum_{l=1}^{T} \sum_{i=1}^N\frac{1}{TN} \mathbb{E}_{\theta_{l-1} \sim P_{l-1}}  \left[-J_{\mathcal{M}_{l,i}}(\pi_{\theta_{l-1}}) \right].
\end{align*}
\end{proposition}

\begin{theorem}[PAC-Bayes Bound for EPIC]

\label{thm:main} Under the settings of Algorithm~\ref{alg:mrl-0} and Assumption~\ref{assump:1}, further assume $i$-th finite horizon MDP of task $i$ has a reward that belongs to $[0,1]$. When running Algorithm~\ref{alg:mrl-0}, we update the default policy distribution $P_l$ for every $N$-th task by using pairs of $\set{(P_l, \barb{P}_l)}_{l=0}^{T-1}$. Let $\mathscr{T}\coloneqq \mathcal \prod_{l=1}^{T} \left( \theta_{l-1} \times \mathcal{M}_l \right)$, and let the expected loss over joint policy and joint trajectory space be:
\vspace{-5pt}$$ \frac{1}{K}\sum_{i=1}^K 
\mathbb{E}_{\left\{\theta_l \right\}_{l=0}^{T-1} \sim P} [\mathbb{E}_{\left\{\tau_l \right\}_{l=1}^{T} \sim \mathscr{T}} [- J_{\mathcal{M}_i}(\pi_{\theta}) ] ].$$\vspace{-5pt} and let the training error be: 
$$\frac{1}{K} \sum_{i=1}^{K} 
\mathbb{E}_{\left\{\theta_l \right\}_{l=0}^{T-1} \sim P} [- J_{\mathcal{M}_i}(\pi_{\theta}) ].$$
Then with probability at least $1-2\exp{-K^{\gamma}}$, for any $0 < \gamma <1$, we have  
\begin{equation}
\label{eq:pac_bound}
    \begin{aligned}
    &\text{expected loss} \le  \text{training error} + \mathscr R(\mathbb{D}_{KL}(P \| \barb{P})), \\ 
    \text{with} \\
    &  \mathscr R(\mathbb{D}_{KL}(P \| \barb{P}))\\ & \coloneqq   
    \frac{2N^{1/2}H \frac{\lambda r}{1-\alpha}  \sqrt{ \frac{1 - \alpha^{2(K/N-1)}}{s_{\min}(1-\alpha^2)}  }}{K^{1/2}}  + \frac{2N^{1/2}H }{K^{(1-\gamma)/2}}.
    \end{aligned}
\end{equation}
\end{theorem} 

\raggedbottom 
The proof is deferred to Appendix \S \ref{app:proofs}.

\paragraph{Remarks.} 
(1) The tightness of the bound depends on the number of lifelong tasks encountered so far $K$, the number of tasks memorized $N$,  the trajectory length  $H$, and the KL divergence between $P$ and $\barb{P}$. By letting $\gamma = 1/4$, the difference of training and generalization error is in the order $\mathcal{O}(K^{-3/8})$.
\ (2) The $N$ appears on the right hand side can be understood as the memory size kept in the agent before it refreshes. The larger $N$ will reduce $K/N$ thus could potentially decrease the first term by making $1 - \alpha^{2K/N-2}$ smaller. However, it also increase the value $N^{1/2}$. According to the experimental results with different seeds, we observe the 
$N=25$ performs well and is robust, we recommend as an initial value.

Practically, a strategy to adaptively adjust $N$ by minimizing the $U(P)$ using a neural network can work.

Overall, this theory enables our learner to optimize the right-hand side of Equation~\eqref{eq:pac_bound} and learn the lifelong policy distribution with a guaranteed minimal true cost.

There are several unsolved questions before we propose a practical lifelong RL algorithm. Theorem~\ref{thm:main} holds for policy distribution $P$ and prior distribution $\barb{P}$ parameterized by $\theta$. However, in practice, determining suitable distributions for $\barb{P}$ and $P$ becomes a crucial challenge. Additionally, computing the posterior distributions $\set{P_{l}}_{l=0}^{T-1}$ is non-trivial. Moreover, we need to identify the appropriate optimization method to learn parameters for $P$. To address these questions, the next two sections will provide solutions and propose a practical lifelong RL algorithm based on the proposed Algorithm \ref{alg:mrl-0}. 

\subsection{Posterior Distribution and Prior Distribution}
\label{post-prior-dist}
In Equation~\eqref{eq:pac_bound}, $P_l$ represents the posterior distribution of $\theta_l$. To optimize the posterior $P_l$, we need to choose appropriate hyperparameters for its distribution. For instance, in a Gaussian distribution, we optimize its mean $\mu$ and variance $\sigma^2$.

Let $\tau_l \sim \theta_{l-1} \times \gM_1 \dots \times \dots \times \gM_N$ be the data induced from previous $\theta_{l-1}$, and define the likelihood function $p(g(\tau_l)|\theta_{l-1})$. Suppose the prior distribution is a probability density function $\barb{p}(\theta_{l-1}; q)$, parametrized by $q$, such as a Gaussian prior $\barb{P}_l = \gN(\barb{\mu}_l, \barb{\sigma}_l)$, where $q \coloneqq  (\barb{\mu}_l, \barb{\sigma}_l)$.

Based on Bayes' Rule, the posterior distribution is uniquely given by $p(\theta_{l-1}|g(\tau_l);q) = \frac{ p(g(\tau_l)|\theta_{l-1})\barb{p}(\theta_{l-1}; q)}{c(q)}$, where $c(q)$ is the normalization constant depends on $q$. We can optimize the hyperparameter $q$ using the following equation:
\beq
\label{eq3}
\min_{q}U_l(P;q)  \coloneqq &  \sum_{i=1}^N \Ept{\theta_{l-1}\sim p(\theta_{l-1}|g(\tau_l);q)}{ -J_{\mathcal{M}_{l,i}}(\pi_{\theta_{l-1}})}\\ & + \mathscr R(\mathbb{D}_{KL}\left(P \| \barb{P}; q\right)),  
\beq
 where  $\mathscr R(\mathbb{D}_{KL}\left(P \| \barb{P}; q\right))$ is defined in \eqref{eq:pac_bound}.
In Equation~\eqref{eq3}, the posterior distribution is unknown, and obtaining an explicit expression requires knowing the data likelihood. In the RL regime, data samples consist of states, actions, and value functions, and one approach is to use the exponential of the negative squared temporal difference (TD) error as an unnormalized likelihood, as suggested in \cite{dann2021provably}, which is left for future research.

In PAC-Bayes, the prior and posterior distributions can belong to different families. However, it is often practical to consider them belonging to a common distribution family, as it simplifies the computation of KL divergence. Hence, we assume the default and prior policy distributions for $\theta$ to be $d$-dimensional Gaussians with unknown parameters $\theta \sim \gN(\mu, \sigma^2 )$. These parameters are updated by minimizing the upper bound.

Based on Equation~\eqref{eq3}, we solve the following problem where $\phi(\theta; \mu, \sigma)$ is the Multivariate Gaussian PDF:
\begin{flalign}
\label{eq20}
 \min_{\mu, \sigma} U(P; \mu, \sigma) \coloneqq & \sum_{i=1}^N \int -J_{\mathcal{M}_{l,i}}(\pi_{\theta_{l-1}})\phi(\theta;\mu, \sigma) d\theta & \nonumber \\
&+ \mathscr{R} \left( \kld{\gN(\mu, \sigma^2)}{\gN(\barb{\mu}; \barb{\sigma^2})} \right). &
\end{flalign}

Evaluating the integral in Equation~\eqref{eq20} analytically is intractable in practice. Therefore, we resort to Monte Carlo Methods, where we sample ${\theta_{l-1,j}}_{j\in[M]}$ to approximate the gradient descent updates by:
\begin{equation}
\label{eq9}
\begin{aligned}
& \nabla_{\mu, \sigma} \hat{U}(P, \{\gM_i\}_{i\in[N]}, \{\theta_{l-1,j}\}_{j\in[M]}; \mu, \sigma)  \\ 
 & \coloneqq 
\frac{1}{M} \sum_{j=1}^M \sum_{i=1}^N -\nabla_{\mu, \sigma} \big \{J_{\mathcal{M}_{l,i}}(\pi_{\theta_{l,j}}) \\ &  + \mathscr{R} \left(\kld{\gN(\mu, \sigma^2)}{\gN(\barb{\mu}; \barb{\sigma^2})}\right) \big \},
\end{aligned}
\end{equation}
where $\theta_{l-1,j}$ is a sample drawn from $P_{l-1}$ to perform gradient descent during optimization in each iteration.

Moreover,  to ensure the parameters $\mu$ and $\sigma$ can be updated, we use indirect sampling by first sampling a multivariate standard normal distribution $\epsilon_j$. The randomness of the parameter $\theta_j$ is then defined as:
\beq 
\label{eq10}
\theta_j = \mu  + \sigma \odot \epsilon_j, \quad \epsilon_j \sim \mathcal{N}(0,I_d).
\beq 
According to Equation \eqref{eq10}, the parameter $\theta_j$ is multivariate normal distributed with $\theta_j \sim \mathcal{N}(\mu, \sigma^2)$.
\subsection{A Practical EPIC Algorithm}

We propose a practical EPIC algorithm, called EPICG, as presented in Algorithm~\ref{alg:mrl}. In this algorithm, a policy is defined as a Gibbs distribution in a linear combination of features: $\pi_{\theta}(s,a) = \frac{\exp{\theta^{\top} \psi_{s,a}}}{\sum_b\exp{\theta^{\top} \psi_{s,b}}}$. Here, $\theta$ can be replaced by a neural network.

For the parameterization of $\theta$ using a neural network, we provide the details in  Appendix \S\ref{a300}.
In each iteration, the agent samples a set of policies ${\theta_j}_{j\in[M]} \sim P$ from the "posterior" policy distribution for every $N$ tasks (Lines 8-9). It then rolls out a set of trajectories $\tau = {\tau}_{i\times j \in [N]\times [M]}$ for each task and estimates the cost (Lines 10-11). More specifically, an action $a$ is sampled as $a \sim \pi_{\theta_{l,j}}(s,a)$, and a state $s$ is sampled using a transition kernel determined by task $\gM_{i}$.

The gradient is taken with respect to the objective function $\hat{U}$ 
(Equation~\eqref{eq20}) which with respect to the cost function  
and with respect to the KL divergence function expressed in Equation~\eqref{eq:pac_bound}.

EPICG uses gradient descent in the space of $P$ to find the policy that minimizes the expected loss, i.e., $P^* \in \arg\inf_{P \in \Pi}  \frac{1}{K}\sum_{i=1}^K \mathbb{E}_{\left\{\theta_l \right\}_{l=0}^{T-1} \sim P} [\mathbb{E}_{\left\{\tau_l \right\}_{l=1}^{T} \sim \mathscr{T}} [- J_{\mathcal{M}_i}(\pi_{\theta}) ] ]$, where we assume the $P^*$ exists. 
\raggedbottom
\begin{algorithm}[t]
   \caption{Empirical PAC-Bayes that Improves Continuously Under Gaussian Prior) (EPICG)}
   \label{alg:mrl}
\begin{algorithmic}[2]
   \STATE {\bfseries Input:} policy dimension $d$; 
    learning rate $\beta$; update frequency $N$; failure probability $\delta$;
    the number of steps allowed in each task $H$; prior evolving speed $\lambda$
   \STATE Initialize prior policy mean and derivation $\barb{\mu}_0$, $\barb{\sigma}_0 \in \mathbb{R}^d$
   \STATE Initialize default policy mean and derivation $\mu_0 \gets \barb{\mu}_0$, $\sigma_0 \gets \barb{\sigma}_0$  
   \FOR{$i=1,2,3,\cdots, K, \cdots \infty $}
   \STATE Receive a new task $\mathcal{M}_i \sim \mathcal{D}_i$ and store it in memory.
   \IF{$i \bmod N = 0$}
    \STATE Let $l = i/N$
        \STATE  Sample $\set{\theta_{l-1,j}}_{j\in[M]} \sim \gN(\mu_{l-1}, \sigma_{l-1}^2)$ by sample $\epsilon_j \sim \mathcal{N}(0,I_d)$
        \STATE Set initial policy, i.e., initialize parameters for neural network $\theta_{l-1,j} \gets \{\mu_{l-1} + \epsilon_j \odot \sigma_{l-1} \}$
    \STATE{Roll out trajectories $\tau_{k,j}$ using $\set{\pi_{\theta_{l-1,j}}}_{j\in[M]}$ and $\set{\gM_k}_{k=i-N+1}^i$ } and store into $\tau_{l,j}$
        \STATE \# Update default and Prior parameters by using $\tau_{l,j}$
        \STATE {\small{$\mu_l \gets \mu_{l-1} - \beta \nabla_{\mu} \hat{U}(P, \{\gM_k\}, \{\theta_{l-1,j}\}_{j\in[M]}; \mu, \sigma)$}}
        \STATE \small{$\sigma_l \gets \sigma_{l-1} - \beta \nabla_{\sigma} \hat{U}(P, \{\gM_k\}, \{\theta_{l-1,j}\}_{j\in[M]}; \mu, \sigma)$}
        \STATE $\barb{\mu}_l \gets (1-\lambda) \barb{\mu}_l + \lambda \mu_l;$ $\barb{\sigma} \gets (1-\lambda) \barb{\sigma}_l + \lambda \sigma_l$
        \STATE {Empty Memory by clearing dataset $\tau_{l}$}
   \ENDIF
   \ENDFOR
\end{algorithmic}
\end{algorithm}

\begin{figure*}[t]
    \centering
    \begin{subfigure}{0.15\textwidth}
        \includegraphics[width=\textwidth]{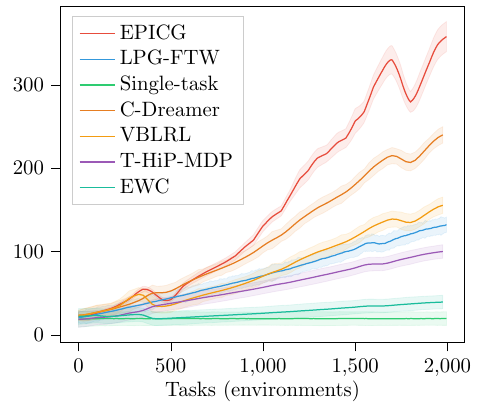}
        \caption{HalfC.-gravity}
    \end{subfigure}
    \begin{subfigure}{0.15\textwidth}
        \includegraphics[width=\textwidth]{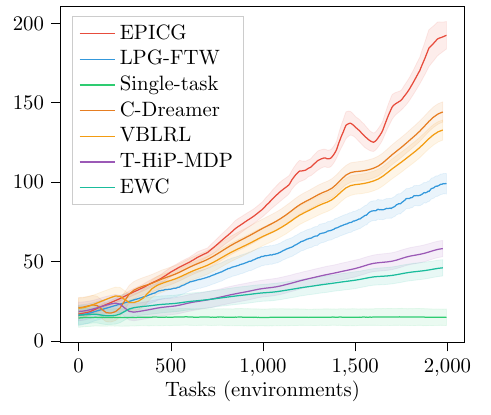}
        \caption{HalfC.-bodyp.}
    \end{subfigure}
    \begin{subfigure}{0.15\textwidth}
        \includegraphics[width=\textwidth]{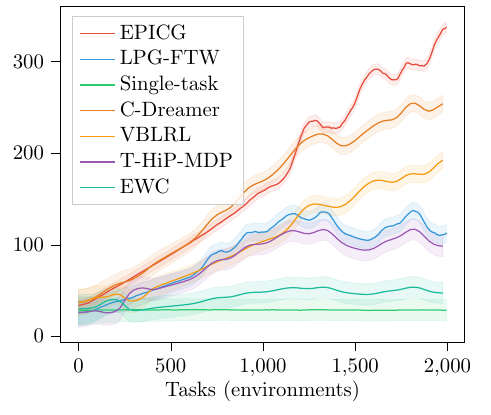}
        \caption{Hopper-gravity}
    \end{subfigure}
    \begin{subfigure}{0.15\textwidth}
        \includegraphics[width=\textwidth]{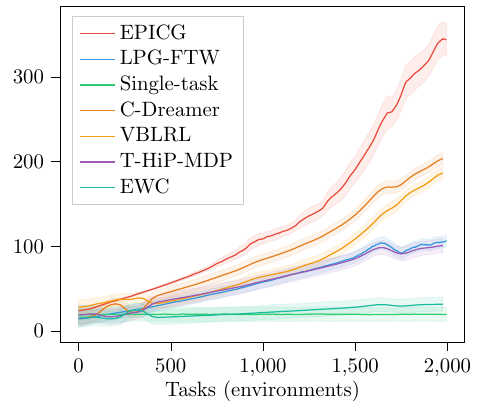}
        \caption{{\small Hop.-bodyp.}}
    \end{subfigure}
    \begin{subfigure}{0.15\textwidth}
        \includegraphics[width=\textwidth]{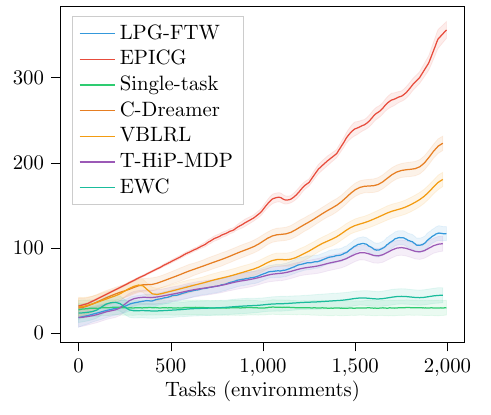}
        \caption{Walker-gravity}
    \end{subfigure}
    \begin{subfigure}{0.15\textwidth}
        \includegraphics[width=\textwidth]{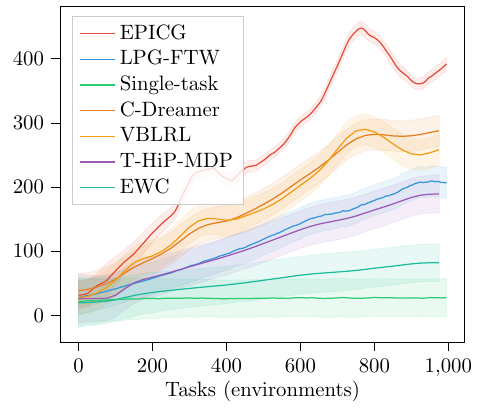}
        \caption{Walker-bodyp.}
    \end{subfigure}
    \caption{Comparison between EPICG and baselines on lifelong RL benchmarks. $X$-axis: tasks, $Y$-axis: reward.  CartPole-Goal with $x_{goal} \sim \mathcal{N}(0,0.1)$ and $x_{goal} \sim \mathcal{N}(0,0.5)$, LunarLander, CartPole-Mass with $\mu_c = 0.5$ and $\mu_c = 1.0$, and Swimmer.}
   \label{fig:combined_tasks}
\end{figure*}

 We denote the optimal expected return as $J^* \coloneqq \frac{1}{K}\sum_{i=1}^K \mathbb{E}_{\left\{\theta_l \right\}_{l=0}^{T-1} \sim P^*} [\mathbb{E}_{\left\{\tau_l \right\}_{l=1}^{T} \sim \mathscr{T}} [ J_{\mathcal{M}_i}(\pi_{\theta}) ] ].$ We denote the return as $\tilde J \coloneqq \frac{1}{K} \sum_{i=1}^{K} 
\mathbb{E}_{\left\{\theta_l \right\}_{l=0}^{T-1} \sim P} [ J_{\mathcal{M}_i}(\pi_{\theta}) ].$, which is also equal to the negative training error. We provide the sample complexity of EPICG. 

\begin{theorem}[Sample Complexity]
\label{lemma:sample complexity}
Consider the setting of Theorem~\ref{thm:main}. 
Given a small $\epsilon > 0$, 
if the number of tasks $K$ satisfies 
\begin{align*}
    K = & \max \left( \frac{16NH^2 \lambda^2 r^2}{s_{\min}(1-\alpha)^3 (1+\alpha)\epsilon^2} , \left( \frac{16NH^2}{\epsilon^2} \right)^{\frac{1}{1-\gamma}} \right) \\ &+ \widetilde{\mathcal O}(N\epsilon^{-4}),
\end{align*}
 then with high probability, $$J^* -  \tilde J \leq \mathcal O(\epsilon),$$ where $\widetilde{\mathcal O}(\cdot)$ suppresses logarithmic dependence.
\end{theorem}
\begin{proof}
The central part of the proof is Theorem~\ref{thm:main} (detailed proof in Appendix~\S\ref{sec_a200}), and the remaining parts are provided in Appendix~\S\ref{lemma_sample_complexity_pg}. 
\end{proof}

Algorithm~\ref{alg:mrl} learns a general policy distribution. However, if we are interested in a policy for any specific task, we can sample a policy from $P$ and use an appropriate single-task learning method to fine-tune the policy. Hence, every task gets a "customized" policy. We also provide an Algorithm~\ref{alg:mrl-SAC} to reflect this in Section~\ref{sub_refinement}.
 
\begin{figure*}[t]
\centering
\minipage[c]{0.16\textwidth}
\subcaptionbox{{\small \centering CartP.-\centering Unif.}\label{fig:cart1}}{
\includegraphics[width=\linewidth, height=1in]{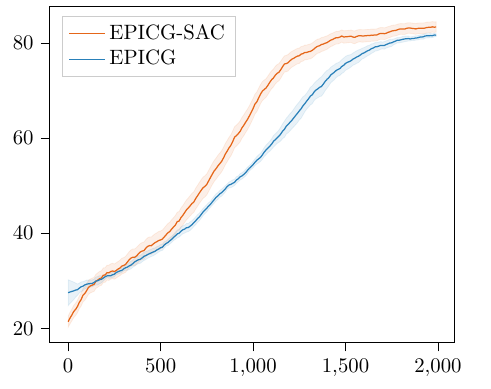}}
\endminipage\hspace{-5pt}
\minipage[c]{0.16\textwidth}
\subcaptionbox{ {\centering LunarL.-GMM}\label{fig:lunar1}}{
\includegraphics[width=\linewidth, height=1in]{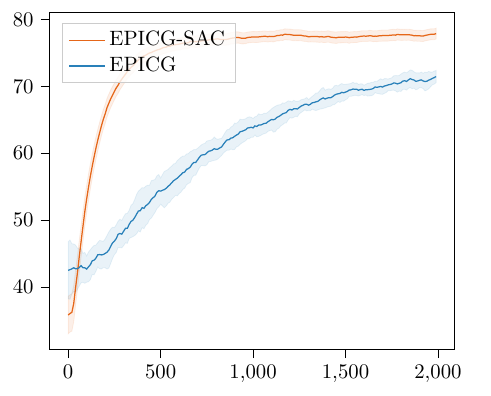}}
\endminipage\hspace{-5pt}
\minipage[c]{0.16\textwidth}
\subcaptionbox{ {\centering AntD.-Unif.}\label{fig:ant1}}{
\includegraphics[width=\linewidth, height=1in]{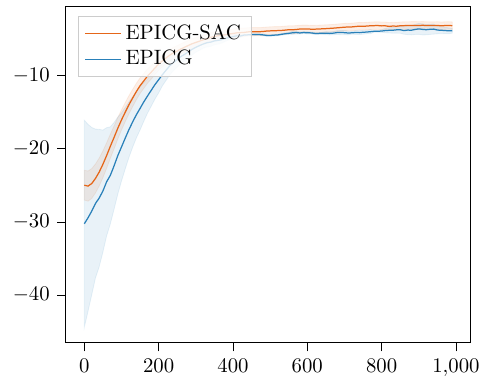}}
\endminipage\hspace{-5pt}
\minipage[c]{0.16\textwidth}
\subcaptionbox{{\centering AntF.B.-Bern.}\label{fig:ant2}}{
\includegraphics[width=\linewidth, height=1in]{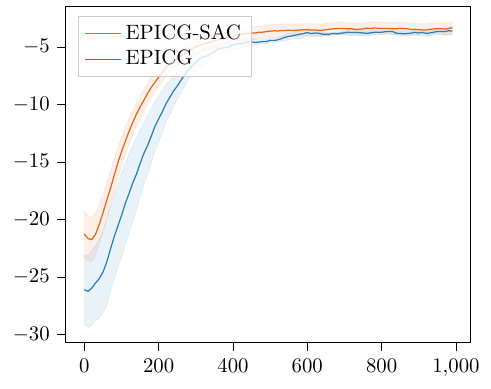}}
\endminipage\hspace{-5pt}
\minipage[c]{0.16\textwidth}
\subcaptionbox{{\centering Swi.-Unif.}\label{fig:swimmer1}}{
\includegraphics[width=\linewidth, height=1in]{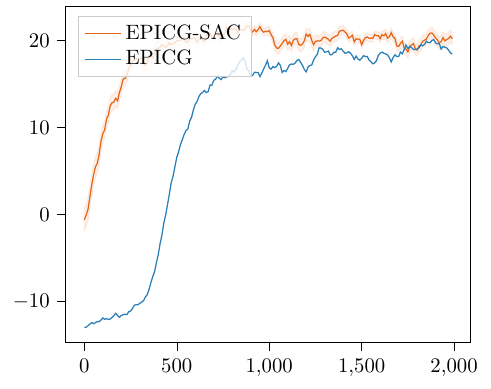}}
\endminipage\hspace{-5pt}
\minipage[c]{0.16\textwidth}
\subcaptionbox{{\centering Hum.D.-Unif.}\label{fig:human}}{
\includegraphics[width=\linewidth, height=1in]{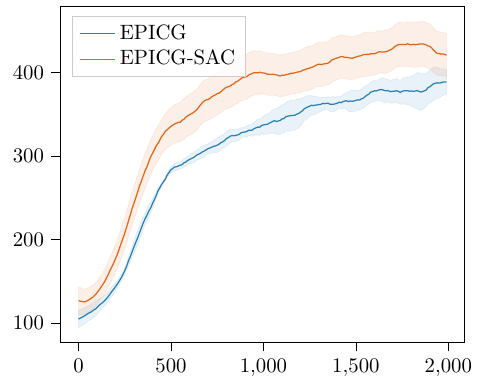}}
\endminipage
\\
\caption[]{Average reward obtained by EPICG-SAC and EPICG in different environments with different lifelong learning settings.
$X$-axis: tasks, $Y$-axis: reward.}
\label{fig:with_sac}
\end{figure*}

\section{EXPERIMENTS}

\subsection{Experimental Setup}
We experiment with common tasks in lifelong-RL benchmarks used in prior works \citep{mendez2020lifelong, fu2022model}, including HalfCheetah-gravity, HalfCheetah-bodyparts, Hopper-gravity,
Hopper-bodyparts, Walker-gravity, Walker-bodyparts. To increase the diversity of lifelong environments, we also create several more lifelong environments,  Cartpole-GMM, LunarLander-Uniform, Ant-Direction-Uniform, Ant-Forward-Backward-Bernoulli, Swimmer-Uniform, Humanoid-Direction-Uniform. Details about the above environments can be found in Appendix \S\ref{openai mamujoco} and in Table \ref{tab:dist_dynamic_cartpole}.
In each lifelong environment, the agents are tested across 2,000 or 1,000  tasks. Each environment has a distinct maximum $H$. As the sequence of 2,000 or 1,000 tasks unfolds, we update the default policy every $N$ tasks. The effectiveness of our approach is assessed by how how fast it learns to maximize return as new tasks emerge.

\subsection{Effective Lifelong Learning}
Figure~\ref{fig:combined_tasks} evaluates EPICG\footnote{Our method is publicly available at \url{https://zzh237.github.io/EPIC/}.} across several control tasks, all of which are lifelong-RL benchmarks  used in prior works \citep{mendez2020lifelong, fu2022model}. This reveals EPICG's noticeable advantage in most scenarios. EPICG consistently outperforms others.  We compared EPICG  against: 1. Continual Dreamer \citep{kessler2023effectiveness}, state-of-the-art lifelong RL method, 2. VBLRL \citep{fu2022model}, Model based Bayesian lifelong RL method; 3. LPG-FTW \citep{mendez2020lifelong},  a lifelong RL method which assumes a factored representations of the policy parameter space; 4. EWC \citep{kirkpatrick2017overcoming},
which is a single-model lifelong RL algorithm that achieves comparable performance with LPG-FTW
as shown in the latter paper; 5. T-HiP-MDP \citep{killian2017robust}, which is a model-based lifelong RL baseline; and 6. Single-Task RL, which let the agent learn the task policy from scratch for every new task and does not use the world policy to help learning. Further details on each baseline can be found in Appendix \S\ref{sec_related_works}.

\subsection{Further Improvement}
\label{sub_refinement}
 EPICG effectively learns a shared distribution $P$ of policy parameters for different tasks. 
 Upon receiving new tasks, we learn the policy distribution by using the sampled policy parameter $\theta$. At this point, this approach has already shown effectiveness in our lifelong learning setting.
Additionally, we can further improve this $\theta$ by optimizing it using data from the new task, customizing it for that particular task. Below we introduce Algorithm \ref{alg:mrl-SAC} (EPICG-SAC), which integrates the EPICG framework with the single task algorithm Soft Actor Critic. 

\begin{algorithm}[!htbp]
   \caption{EPICG-SAC}
   \label{alg:mrl-SAC}
\begin{algorithmic}[1]  
   \STATE {\bfseries Input:} Same setting as Algorithm~\ref{alg:mrl}
   \FOR{$i=1,2,3,\cdots, K, \cdots \infty$}
       \STATE Receive a new task $\mathcal{M}_i$.
       \IF{$i \bmod N = 0$}
           \STATE Do EPICG Policy Distribution Learning.
       \ENDIF
       \STATE \bf{SAC single-task train-eval loop.}
   \ENDFOR
\end{algorithmic}
\end{algorithm}

We then compare  EPICG-SAC and EPICG for different environments.  Figure~\ref{fig:with_sac} shows EPICG-SAC achieves faster learning than EPICG.

\subsection{Ablation on KL divergence regularization}
We first verify the empirical performance when we add the regularizer in \eqref{eq9} compared to having no regularizer by a comparison study.  The results are shown in Figure~\ref{fig:ab1}, where we observe that adding the regularizer facilitates fast adaptation, leads to learning a higher reward, and also reduces the variance, which leads to a more stable learning compared to having no regularize.

\begin{figure}[ht]
\centering
\hspace*{-0.0in}

\begin{subfigure}[c]{0.23\textwidth}
\centering
\includegraphics[width=\linewidth,height=1.in]{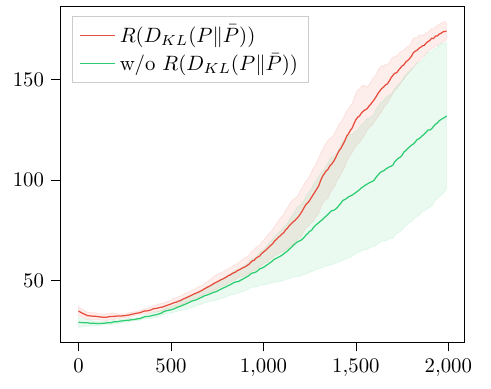}
\caption{Cartpole-GMM}
\label{fig:ab1-2}
\end{subfigure}\hspace{-8pt}
\begin{subfigure}[c]{0.23\textwidth}
\centering
\includegraphics[width=\linewidth,height=1.in]{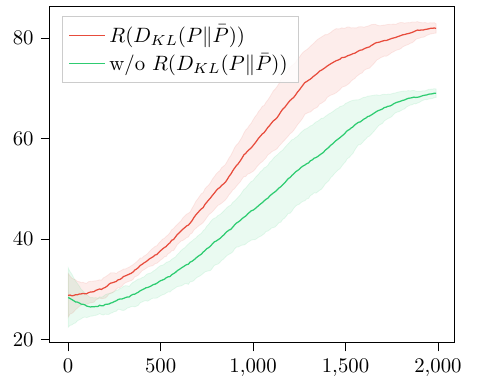}
\caption{LunarLander-GMM}
\label{fig:ab1-1}
\end{subfigure}
\caption[]{Comparison of adding $\mathscr{R}(\mathbb{D}_{KL}(P \| \bar{P}))$ vs. not. $X$-axis: tasks, $Y$-axis: reward.}
\label{fig:ab1}
\end{figure}

\vspace{-5pt}
\subsection{Experiments on Memory Size $N$}

\begin{figure}[ht]
\centering

\begin{subfigure}{0.15\textwidth}
\centering
\includegraphics[width=\linewidth]{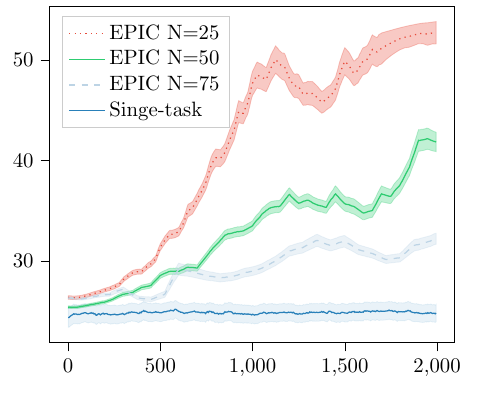}
\caption{}
\label{cart_n}
\end{subfigure}
\begin{subfigure}{0.15\textwidth}
\centering
\includegraphics[width=\linewidth]{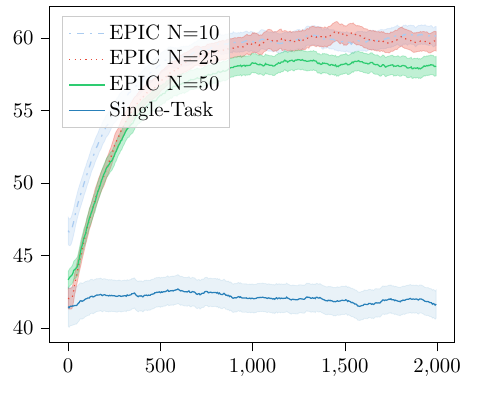}
\caption{}
\label{lunar_n}
\end{subfigure}
\begin{subfigure}{0.15\textwidth}
\centering
\includegraphics[width=\linewidth]{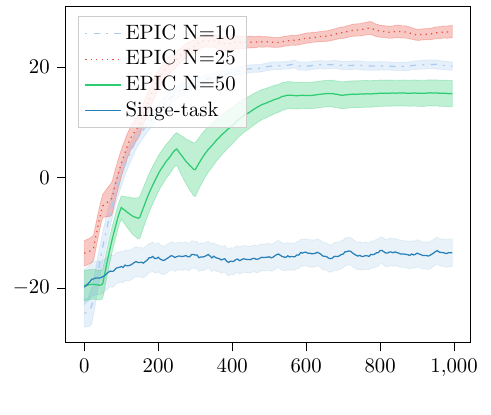}
\caption{}
\label{swimmer_n}
\end{subfigure}

\vspace{-1em}
\caption{Comparison of different update frequency $N$ on 
(\subref{cart_n}) CartPole-Uniform;
(\subref{lunar_n}) LunarLander-Uniform;
(\subref{swimmer_n}) Swimmer-Uniform. $X$-axis: tasks, $Y$-axis: reward}
\label{fig:all_n}
\end{figure}

As we discussed in Theorem~\ref{thm:main}, there is a performance trade-off on the number of tasks $N$ retained
in memory. Experimental results have verified this theoretical finding. We can see that in Figure~\ref{fig:all_n} the practical effect of $N$ on the performance of learning is double-sided.


\section{CONCLUSION AND FUTURE WORKS}
\label{sec:discuss}

\noindent In this work, we address the challenging problem of lifelong RL and propose a novel algorithm, \ourmod{}(G), for distribution learning and policy sampling. Our approach leverages the concept of a \textit{world policy}, a shared policy distribution across tasks. This world policy is updated continuously, enabling our algorithm to handle both non-stationarity and catastrophic forgetting, achieving best-in-class performance across a suite of complex lifelong RL benchmarks.

Future directions include exploring more accurate ways to obtain the posterior distribution of the policy parameters, as discussed in Section~\ref{post-prior-dist}. Additionally, since the optimization objective in Equation \eqref{eq9} is nonconvex, better performance guarantees could be achieved by investigating multiple optimizations across tasks.

\bibliographystyle{apalike}
 
\bibliography{example_paper}


\clearpage 

\onecolumn
\begin{center}
    \textbf{\Large Supplementary Material}
\end{center}

\appendix



\section{Detailed Proofs}
\label{app:proofs}

\subsection{Proof of Proposition~\ref{cor_docomp}}
\label{sec_proof1}
\begin{*proposition}[Decomposition of Training Error]
    Assume Assumption~\ref{assump:1} holds, then it holds true that
\begin{align*}
   &\frac{1}{K} \sum_{i=1}^{K} 
\mathbb{E}_{\left\{\theta_l \right\}_{l=0}^{T-1} \sim P} [- J_{\mathcal{M}_i}(\pi_{\theta}) ] = \frac{1}{T}\sum_{l=1}^{T} \frac{1}{N}\sum_{i=1}^N \mathbb{E}_{\theta_{l-1} \sim P_{l-1}}  \left[-J_{\mathcal{M}_{l,i}}(\pi_{\theta_{l-1}}) \right].
\end{align*}

\end{*proposition}
\begin{proof}
By Assumption~\ref{assump:1},
\beq 
\label{eq7}
& P(\theta_{T-1}, \dots, \theta_{0}) \\ =& P(\theta_{T-1}|\theta_{T-2},\dots, \theta_0)\times \dots \times P(\theta_1|\theta_0)P(\theta_{0})\\
=& P(\theta_{T-1}|\theta_{T-2})\times \dots \times P(\theta_1|\theta_0)P(\theta_{0}) \\
\coloneqq & P_{T-1}\times \dots \times P_l \times \dots \times P_{0}
\beq 

By law of total expectation, we have 
\begin{align*}
 &\mathbb{E}_{\left\{\theta_l \right\}_{l=0}^{T-1} \sim P} [- J_{\mathcal{M}_i}(\pi_{\theta}) ] = \mathbb{E}_{\left\{\theta_l\sim P_l \right\}_{l=0}^{T-1} } [- J_{\mathcal{M}_i}(\pi_{\theta}) ] \\ 
&= \frac{1}{N} \sum_{i=1}^N E_{P_{T-2}, \dots, P_0}E_{\theta_{T-1} \sim P_{T-1}|P_{T-2}, \dots, P_0}[- J_{\gM_{T,i}}(\pi_{\theta_{T-1}})] +   \\
&\cdots + \frac{1}{N} \sum_{i=1}^N E_{\theta_{0} \sim P_0}[- J_{\gM_{T,i}}(\pi_{\theta_0})] \\
&= \frac{1}{N} \sum_{i=1}^N E_{P_{T-2}}E_{\theta_{T-1} \sim P_{T-1}|P_{T-2}}[- J_{\gM_{T,i}}(\pi_{\theta_{T-1}})] +   \\
&\cdots + \frac{1}{N} \sum_{i=1}^N E_{\theta_{0} \sim P_0}[- J_{\gM_{T,i}}(\pi_{\theta_0})] \\ 
&= \frac{1}{N} \sum_{i=1}^N E_{\theta_{T-1} \sim P_{T-1}}[- J_{\gM_{T,i}}(\pi_{\theta_{T-1}})] +   \\
&\cdots + \frac{1}{N} \sum_{i=1}^N E_{\theta_{0} \sim P_0}[- J_{\gM_{T,i}}(\pi_{\theta_0})] \\ 
&=\frac{1}{T}\sum_{l=1}^{T} \frac{1}{N}\sum_{i=1}^N \mathbb{E}_{\theta_{l-1} \sim P_{l-1}}  \left[-J_{\mathcal{M}_{l,i}}(\pi_{\theta_{l-1}}) \right].
\end{align*}
   
\end{proof}

\subsection{Azuma-Hoeffiding or Freedmans inequality for martingale difference sequences for RL}
\label{sec_a200}
We let the Algorithm~\eqref{alg:mrl-0} experience $K$ tasks, for every $N$ task Algorithm~\eqref{alg:mrl-0} performs lifelong learning, i.e., learns the posterior distribution hyperparameters for policy. 

Remember, in Algorithm~\eqref{alg:mrl-0}, where the the lifelong setting happens with $K$ tasks streaming in, we update the default policy every $N$ tasks and estimate the training cost based on the most recent $N$ tasks. 
 The entire learning process consists of a total of $T$ episodes of updates.

Let the distribution of policy has a parameter mean $\mu$ and variance-covariance $\sigma^2$ for illustration purposes. For each $l \in [T]$ episode, Algorithm~\eqref{alg:mrl-0} proceeds as: 
\begin{enumerate}[topsep=0pt, itemsep=-1ex]
\item Sample $\theta_{l-1} \sim P_{l-1}(\theta; \mu_{l-1}, \sigma_{l-1}^2)$ learnt from the previous episode.
\item Using $\theta_{l-1}$  and $N$ tasks $\set{\gM_{i,l}}_{i\in[N]}$ from current episode to collect data $\tau_l$; 
\item Using an optimization algorithm and data from Step $1$ to learn the posterior distribution $P_l(\theta; \mu_l, \sigma_l^2)$'s hyper parameters; 
\end{enumerate}

For $l$'s episode, for policy $\pi_{\theta_l}$, which is learned using data roll-out by $\theta_{l-1} \sim P_{l-1}(\theta; \mu_{l-1}, \sigma_{l-1}^2)$, its value function with task $\gM_{l,i}$ is defined as the total discounted expected rewards,    
\beq 
\label{valuereward}
V_{\gM_{l,i}}^{\pi_{\theta_l}}(s_{l,1}) \coloneqq  J_{\mathcal{M}_{l,i}}(\pi_{\theta_l}) = \ept{\sum_{h=1}^{H-1}\gamma^{h-1}r_{l,h}|\pi_{\theta_l}, s_{l,1}, \gM_{l,i}}, 
\beq
from a length of $H$ consecutive sample transitions,  $s_{l,1}, a_{l,1}, r_{l,1}, s_{l,2}, a_{l,2}, r_{l,2}, \dots, s_{l,H} \sim \pi_{\theta_l} \times \gM_{l,i}$.
Note, $\theta_l$ is a function of $\tau_l, \mu_{l-1}, \sigma_{l-1}^2$, which is random, and the randomness is dependent on $\tau_l, \mu_{l-1}, \sigma_{l-1}^2$. 

The posterior distribution $P(\theta; \mu, \sigma^2)$ defines a randomized $\theta$.
Algorithm~\eqref{alg:mrl-0} draws a $\theta$ according
to $\theta \sim P(\theta; \mu, \sigma^2)$ at each round of the whole process and applies it to learn the hyperparameter of $P(\theta; \mu, \sigma^2)$ on the next round. For notation-wise, if we do not use subscript $l$, it means the statement holding for general. 

For any $\theta$, let $S_T$ be the difference between the expected and empirical objective value function after the $T$-th round, 
\beq 
\label{mtg}
S_T \coloneqq \sum_{l=1}^T D_l,  \text{\quad $l\in [T]$},
\beq
where 
$D_l \coloneqq \sum_{i=1}^N \Ept{\tau_l \sim \theta_{l-1} \times \gM_{l}}{V_{\gM_{l,i}}^{\pi_{\theta_l}}(s_{l,1})|\gF_{l-1}} - \sum_{i=1}^N V_{\gM_{l,i}}^{\pi_{\theta_l}}(s_{l,1})$.
And the filtration $\gF_{l-1} = \sigma(\set{\theta_k}_{k\le l-2}, \set{\gM_{i,k-1}}_{i\in[N], k\le l-1})$ is the  $\sigma$-algebra generated by the random variables $\set{\theta_k}_{k\le l-2}$ and $\set{\gM_{i,k-1}}_{i\in[N], k\le l-1}$.

We first show that using Algorithm~\eqref{alg:mrl-0}, after $T$-th lifelong learning updates, with probability at least $1-\delta$, for a small $\delta \in (0,1)$, $S_T = \gO(\sqrt{T})$.

\begin{theorem}
\label{azifortask}
    Let $\set{D_l}_{l\le T}$, and $S_T$ be defined in Equation~\eqref{mtg}. For fixed $N$ and $H$, 
    then with probability at least $1-\delta$, 
    \beq
    \label{eq41}
    \abs{S_T}  \lesssim \sqrt{\frac{1}{2}\paren{\ln \frac{2}{\delta}} T N^2H^2}.
    \beq
Furthermore, if $\abs{D_l} \le b$ for all $l\le T$, and let \beq 
    \tilde{S}_T = \sum_{l=1}^T \ept{D_l^2|\gF_{l-1}}, 
    \beq  
    then with probability at least $1-\delta$, for $\lambda \in [0, \frac{1}{b}]$, 
    \beq
    \label{eq49}
    \abs{S_T}  \lesssim \frac{1}{\lambda}\ln \frac{2}{\delta} + \lambda \tilde{S}_T\le \frac{1}{\lambda}\ln \frac{2}{\delta} + \lambda T N^2H^2.
    \beq
\end{theorem}
\begin{proof} Firstly, the $\theta_l$ comes from the posterior distribution $P(\theta; \mu, \sigma^2)$, which depends on $\set{\theta_k}_{k\le l-1}$  and $\set{\gM_{i,l}}_{i\in[N]}$. Furthermore, for a fixed $\theta_l$, we see that $D_l$ is $\gF_{l-1}$ measurable. So given that $D_l = \sum_{i=1}^N \ept{V_{\gM_{i,l}}^{\pi_{\theta_l}}(s_{l,1})|\gF_{l-1}} - \sum_{i=1}^N V_{\gM_{i,l}}^{\pi_{\theta_l}}(s_{l,1})$, we have  $\ept{D_l|\gF_{l-1}} = 0$. And $\set{D_l}_{l\in[T]}$ is a martingale difference sequence of functions of $\theta$.
Furthermore, $\sum_{i=1}^N V_{\gM_{i,l}}^{\pi_{\theta_l}}(s_{T,1}) \le NH$ because of Equation~\eqref{valuereward} and the reward $r_{l,h}$ is in $[0,1]$ in Theorem~\ref{thm:main}. And $\set{D_l}_{l\in[T]}$ is a bounded martingale difference sequence. In other words, $D_l \in [a_l,b_l]$, with $a_l = -NH$, $b_l = NH$, and $b = NH$, this is true because in Equation~\eqref{valuereward}, the reward $r_{l,h}$ is in $[0,1]$ in Theorem~\ref{thm:main}. 
Then based on the conclusion of Theorem~\ref{thm:ahi}, the Azuma-Hoeffding Inequality for bounded martingale difference sequence, we get the result in Equation~\eqref{eq41}.

For Equation~\eqref{eq49}, we use Theorem~\ref{lm71}, the Freedmans Inequality for martingale difference sequence. For any $\lambda \in [0, \frac{1}{b}]$, where $\abs{D_l} \le b$, we have 
\beq 
&\sP\set{S_T \ge t|\gF_{T-1}} = \sP\set{e^{\lambda S_T} \ge e^{\lambda t}|\gF_{T-1}}
\le e^{-\lambda t} \ept{e^{\lambda S_T}|\gF_{T-1}} \\
&\le e^{-\lambda t} e^{\lambda^2 \tilde{S}_T} = e^{-\lambda t + \lambda^2 \tilde{S}_T}.\\
\beq 
Thus, 
\beq 
&\sP\set{S_T \ge t} \le  \Ept{\gF_{T-1}}{e^{-\lambda t + \lambda^2 \tilde{S}_T}}.\\
\beq 
Repeating this argument for $-S_T$, we get the same bound, so overall, 
\beq 
&\sP\set{\abs{S_T} \ge t} \le 2\Ept{\gF_{T-1}}{e^{-\lambda t + \lambda^2 \tilde{S}_T}}.
\beq 
Let $\Ept{\gF_{T-1}}{2e^{-\lambda t + \lambda^2 \tilde{S}_T}} = \delta$, we get 
$t = \frac{1}{\lambda} \ln \frac{2}{\delta} + \lambda \tilde{S}_T$. 
Therefore, with probability at least $1-\delta$, for $\lambda \in [0, \frac{1}{b}]$, 
    \beq
    \abs{S_T}  \lesssim \frac{1}{\lambda}\ln \frac{2}{\delta} + \lambda \tilde{S}_T \labelrel\le{le86} \frac{1}{\lambda}\ln \frac{2}{\delta} + \lambda T N^2H^2  ,
    \beq
where \eqref{le86} holds since $D_l\in[a_l, b_l]$ almost surely, the conditioned variable $D_l|\gF_{l-1}$ also belongs to this interval almost surely, then we have  $\tilde{S}_T \le \sum_{l=1}^T \frac{(a_l - b_l)^2}{4} \le TN^2H^2$ by Lemma~\ref{varlemma}.
\end{proof}

Remark of Theorem~\ref{azifortask}. Note, if we minimize the right hand side of Equation~\eqref{eq49} with respect to $\lambda$, we get the optimal $\lambda = \sqrt{\frac{\ln \frac{2}{\delta}}{TN^2H^2}}$, and $\abs{S_T} \lesssim 2\sqrt{\ln \frac{2}{\delta} TN^2H^2} $. Hence, Equation~\eqref{eq49} (derived from Freedmans’s Inequality)
matches Equation~\eqref{eq41} (derived from Azuma-Hoeffding Inequality)
up to minor constants and logarithmic factors in the
general case, and can be much tighter when the variance $\tilde{S}_T = \sum_{l=1}^T \ept{D_l^2|\gF_{l-1}}$ is
small.


\subsection{PAC-Bayes Bound}

The quantity $S_T$ is of our interest as it is the difference between the expected and empirical objective
value after the $T$-th round. In the previous Theorem~\ref{azifortask}, we show that $S_T$ is bounded for the sampled $\theta$. Our Algorithm~\ref{alg:mrl-0} keep updating $P$ and sampling $\theta$, rendering $\set{P_l}$ and $\set{\theta_l}$, $l=0, \dots, T-1$. To abuse the notation, without further reminder, in the following proof, we use $P(\set{\theta_l})$ and $\barb{P}(\set{\theta_l})$ to denote the joint posterior and prior, and use $\theta$ to denote the set for all $\theta_l$, similarly for the hyperparameter in the distribution. 
However, since Algorithm~\eqref{alg:mrl-0} draws $\theta$ from posterior distribution $P(\theta;\mu,\sigma^2)$, we are interested in the expected value of $S_T$, which is $\Ept{P}{S_T}$. At the same time, we will relate all
possible $P(\theta; \mu, \sigma^2)$ to its corresponding “reference”  distribution $\barb{P}$, the prior distribution of $\theta$, which is selected before we do step 3. 
Let $q = (\mu, \sigma)$. In the next Lemma, we will control $\Ept{P}{S_T}$ for any $q$.

\begin{corollary}[Uniform control of all distributions]
\label{ucad} Let $\theta\sim P(\theta; q)$ come from a parametrized distribution with the same family as $\barb{P}$, where $\barb{P}$ is a given prior distribution. 
Let $\set{D_l}_{l\le T}$ and $S_T$ follow the same definition in Theorem~\ref{azifortask}. For any $\lambda > 0$,  let $g(\theta)\coloneqq \lambda S_T(\theta) - \lambda^2 \tilde{S}_T$, where $\tilde{S}_T$ is defined either
\beq 
\label{azs}
\tilde{S}_T = \sum_{l=1}^T \frac{(a_l - b_l)^2}{8}, \quad(\text{Align with Equation~\eqref{eq41}} )
\beq 
 or
\beq 
\label{bers}
\tilde{S}_T = \sum_{l=1}^T \ept{D_l^2|\gF_{l-1}}, \quad(\text{Align with Equation~\eqref{eq49}} )
\beq 

then with probability at least $1-\delta$, and for all $P(\theta; q)$, 
\begin{align}
\label{eq69}
       \abs{ \Ept{P}{S_T(\theta)}} \le \frac{\sD_{KL}(P \| \barb{P}) + \ln \frac{2}{\delta}}{\lambda}  + \lambda \Ept{P}{\tilde{S}_T},
\end{align}
with $\lambda > 0$ for \eqref{azs}, and $\lambda \in [0, \frac{1}{b}]$ required for \eqref{bers}, where $\abs{D_l} \le b$ and $\abs{a_l} \le b, \abs{b_l}\le b$. 
\end{corollary}
\begin{proof}
Based on Theorem~\ref{thmdvv}, the Donsker$-$Varadhans Representation formula, we have 
\begin{equation}
\label{eq147}
\begin{aligned}
&\Ept{P}{\abs{\lambda S_T} - \lambda^2 \tilde{S}_T } \le \sD_{KL}(P \| \barb{P}) + \ln \paren{\Ept{\barb{P}}{e^{\abs{\lambda S_T} - \lambda^2 \tilde{S}_T}}} \\
     \labelrel{\lesssim}{1} & \sD_{KL}(P \| \barb{P}) + \ln \paren{\frac{1}{\delta}\Ept{\set{D_l}_{l\le T}}{\Ept{\barb{P}}{e^{\abs{\lambda S_T} - \lambda^2 \tilde{S}_T}}}}  \quad (\text{with probability at least $1-\delta$})\\
     \le& \sD_{KL}(P \| \barb{P}) + \ln \paren{\frac{1}{\delta}\Ept{\set{D_l}_{l\le T}}{\Ept{\barb{P}}{e^{\lambda \max\set{S_T,-S_T} - \lambda^2 \tilde{S}_T}}}}\\
     \le& \sD_{KL}(P \| \barb{P}) + \ln \paren{\frac{1}{\delta}\Ept{\set{D_l}_{l\le T}}{\Ept{\barb{P}}{e^{\lambda S_T - \lambda^2 \tilde{S}_T} + e^{\lambda (-S_T) - \lambda^2 \tilde{S}_T}}}}\\
     \le& \sD_{KL}(P \| \barb{P}) + \ln \paren{\frac{1}{\delta}\paren{\Ept{\barb{P}}{\Ept{\set{D_l}_{l\le T}}{e^{\lambda S_T - \lambda^2 \tilde{S}_T}} + \Ept{\set{D_l}_{l\le T}}{e^{\lambda (-S_T) - \lambda^2 \tilde{S}_T}}}}}\\
    \labelrel{\le}{2}& \sD_{KL}(P \| \barb{P}) + \ln \frac{2}{\delta}.
\end{aligned}
\end{equation}

Where \eqref{1} is due to the Lemma~\ref{mkv}, the Markov's Inequality. For the $\tilde{S}_k$ defined in Equation~\eqref{azs}, \eqref{2} holds because of Equation~\eqref{azm} in Theorem~\ref{thm:ahi}. For the $\tilde{S}_k$ defined in Equation~\eqref{bers}, it is based on Theorem~\ref{lm71}.

Moving $\lambda \Ept{P}{\tilde{S}_T}$ to the other side of Equation~\eqref{eq147}, and dividing by $\lambda$ from both sides, 
\beq 
       \abs{ \Ept{P}{S_T(\theta)}} \le \frac{\sD_{KL}(P \| \barb{P}) + \ln \frac{2}{\delta}}{\lambda}  + \lambda \Ept{P}{\tilde{S}_T}.
\beq 

\end{proof}

\begin{corollary}[Proof of Theorem~\ref{thm:main}]
If we let $\tilde{S}_T$ follow the definition in Equation \eqref{azs}, for any $\lambda > 0$, 

\beq 
\abs{ \Ept{P}{S_T(\theta)}} \le \sqrt{2}NH\sqrt{T \paren{\sD_{KL}(P \| \barb{P}) + \ln \frac{2}{\delta} + \frac{\ln 2}{2\ln c}\paren{\frac{\mathbb{D}_{KL}\left(P \| \barb{P}\right)}{\ln(\frac{2}{\delta})}+1}}},
\beq 

and if we let $\tilde{S}_T$ follow the definition in Equation \eqref{bers}, for any $\lambda \in [0, \frac{1}{b}]$, where $\abs{D_l}\le b$, 
\beq 
\abs{ \Ept{P}{S_T(\theta)}} \le \min\set{2NH\sqrt{T \paren{\sD_{KL}(P \| \barb{P}) + \ln \frac{\gO(\ln T)}{\delta}}}, 2NH\paren{\sD_{KL}(P \| \barb{P}) + \ln \frac{\gO(\ln T)}{\delta}}}.
\beq
\end{corollary}

\begin{proof}
The next step would be to optimize the $\lambda$ in Equation~\eqref{eq69}, to get the tightest upper bound.
However, the value of $\lambda$ that minimizes Equation~\eqref{eq69} depends on $P$, whereas we would like to have a result that holds for all possible distributions simultaneously, which is not possible. So we do a discretization of $\lambda$. We make a grid of $\lambda$'s value in a form of a geometric sequence and for each value of $\sD_{KL}(P \| \barb{P})$, we pick a value of $\lambda$ from the grid, which is the closest to the one that minimizes the right-hand side of Equation~\eqref{eq69} upon to some minor errors.

First, in Equation~\eqref{eq69},  we get 
\begin{align}
\label{eq208}
   \lambda^* =  \argmin_{\lambda} \frac{\sD_{KL}(P \| \barb{P}) + \ln \frac{2}{\delta}}{\lambda}  + \lambda\Ept{P}{\tilde{S}_T} = \sqrt{\frac{\sD_{KL}(P \| \barb{P}) + \ln \frac{2}{\delta}}{\Ept{P}{\tilde{S}_T}}}. 
\end{align} 

Then putting $\lambda^*$ back to Equation~\eqref{eq69}, we get 
\beq 
\label{eq214}
\abs{ \Ept{P}{S_T(\theta)}} \le 2 \sqrt{(\sD_{KL}(P \| \barb{P}) + \ln \frac{2}{\delta})\Ept{P}{\tilde{S}_T}}.
\beq 
 
Moreover,
$\sD_{KL}(P \| \barb{P}) \ge  0$; note that, if the $\sD_{KL}(P \| \barb{P}) = 0$,
we get
\beq 
\lambda^{**} = \argmin_{\lambda} \frac{ \ln \frac{2}{\delta}}{\lambda}  + \lambda \Ept{P}{\tilde{S}_T} = \sqrt{\frac{\ln \frac{2}{\delta}}{\Ept{P}{\tilde{S}_T}}},
\beq 
putting $\lambda^{**}$ back to Equation~\eqref{eq69} with $\sD_{KL}(P \| \barb{P}) = 0$, we get  
\beq 
\abs{ \Ept{P}{S_T(\theta)}} \le 2 \sqrt{\ln \frac{2}{\delta}\Ept{P}{\tilde{S}_T}}.
\beq 
Here $\lambda^{**}$ is also a lower bound for the $\lambda$.

Now if we let $\tilde{S}_T$ follow the definition in \eqref{bers}, we have $\Ept{P}{\tilde{S}_T} \le Tb^2$ by the definition of $\tilde{S}_T$ in Equation~\eqref{azs} and \eqref{bers}, and $\abs{D_l}\le b$. Thus, we have $\lambda \in \bracket{\frac{1}{b}\sqrt{\frac{\ln \frac{2}{\delta}}{T}}, \min\set{\sqrt{\frac{\sD_{KL}(P \| \barb{P}) + \ln \frac{2}{\delta}}{\Ept{P}{\tilde{S}_T}}},\frac{1}{b}}}$, note for such $\tilde S_T$, we required $\lambda \le \frac{1}{b}$.

However, in the setting $\sD_{KL} (P \| \barb{P}) > 0$, the value of $\lambda$ that minimizing right hand side of Equation \eqref{eq69} is given by Equation~\eqref{eq208}, which depends on $P$, as early mentioned, thus we use the geometric sequence $\set{\lambda_j}_{j=0}^{J-1}$ over the range $\bracket{\frac{1}{b}\sqrt{\frac{\ln \frac{2}{\delta}}{T}}, \frac{1}{b}}$, for $\lambda_j = c^j\frac{1}{b}\sqrt{\frac{\ln \frac{2}{\delta}}{T}}$, for some $c > 1$ and $j = 0, \dots, J-1$. 
Now we have the geometric series satisfy
$ 
c^{J-1}\frac{1}{b}\sqrt{\frac{\ln \frac{2}{\delta}}{T}} \le \frac{1}{b}
$
so as long as 
$ 
J - 1 = \ceil{\frac{1}{\ln c}\ln{\sqrt{\frac{T}{\ln \frac{2}{\delta}}}}}.
$ 

If $\min\set{\sqrt{\frac{\sD_{KL}(P \| \barb{P}) + \ln \frac{2}{\delta}}{\Ept{P}{\tilde{S}_T}}},\frac{1}{b}} = \sqrt{\frac{\sD_{KL}(P \| \barb{P}) + \ln \frac{2}{\delta}}{\Ept{P}{\tilde{S}_T}}}$, so there are at most total $J = \ceil{\frac{1}{\ln c}\ln{\sqrt{\frac{T}{\ln \frac{2}{\delta}}}}} + 1$ $\lambda$'s.

We go back to the proof of Corollary \ref{ucad}, let $\delta = \sum_{j=0}^{J} \delta_j$, with $\delta_j = \frac{1}{J}\delta$, $j\in \sN_{\ge 0}$.

Then for any $\delta_j = \frac{1}{J}\delta$, we have 

\beq
\label{eq251}
\sP&\paren{\Ept{P}{\abs{\lambda S_T} - \lambda^2 \tilde{S}_T } \ge \sD_{KL}(P \| \barb{P}) + \ln \frac{2}{\delta_j} \Bigg| \lambda \in \bracket{\frac{1}{b}\sqrt{\frac{\ln \frac{2}{\delta}}{T}}, \frac{1}{b}} }\\ 
\labelrel{\le}{eq254}& 
\sP \left( \sD_{KL}(P \| \barb{P}) + \ln \paren{\Ept{\barb{P}}{e^{\abs{\lambda S_T} - \lambda^2 \tilde{S}_T}}} \right. \\
& \left. \ge  \sD_{KL}(P \| \barb{P}) + \ln \paren{\frac{1}{\delta_j}\Ept{\set{D_l}_{l\le T}}{\Ept{\barb{P}}{e^{\abs{\lambda S_T} - \lambda^2 \tilde{S}_T}}}}\Bigg| \lambda \in \bracket{\frac{1}{b}\sqrt{\frac{\ln \frac{2}{\delta}}{T}}, \frac{1}{b}}  \right),
\beq
where \eqref{eq254} holds because $\Ept{P}{\abs{\lambda S_T} - \lambda^2 \tilde{S}_T } \le \sD_{KL}(P \| \barb{P}) + \ln \paren{\Ept{\barb{P}}{e^{\abs{\lambda S_T} - \lambda^2 \tilde{S}_T}}}$ almost surely, and  $\sD_{KL}(P \| \barb{P}) + \ln \frac{2}{\delta_j} \ge \sD_{KL}(P \| \barb{P}) + \ln \paren{\frac{1}{\delta_j}\Ept{\set{D_l}_{l\le T}}{\Ept{\barb{P}}{e^{\abs{\lambda S_T} - \lambda^2 \tilde{S}_T}}}}$ almost surely. Here $\Bigg|$ indicates given not conditional on. Following this, we have 

\beq 
\label{eq252}
\sP&\paren{\Ept{P}{\abs{\lambda S_T} - \lambda^2 \tilde{S}_T } \ge \sD_{KL}(P \| \barb{P}) + \ln \frac{2}{\delta_j} \Bigg| \lambda \in \bracket{\frac{1}{b}\sqrt{\frac{\ln \frac{2}{\delta}}{T}}, \frac{1}{b}} }\\ 
\le& 
\sP \left.\Bigg(\sD_{KL}(P \| \barb{P}) + \ln \paren{\Ept{\barb{P}}{e^{\abs{\lambda S_T} - \lambda^2 \tilde{S}_T}}} \right.\\
& \left. \ge  \sD_{KL}(P \| \barb{P}) + \ln \paren{\frac{1}{\delta_j}\Ept{\set{D_l}_{l\le T}}{\Ept{\barb{P}}{e^{\abs{\lambda S_T} - \lambda^2 \tilde{S}_T}}}}\Bigg| \lambda \in \bracket{\frac{1}{b}\sqrt{\frac{\ln \frac{2}{\delta}}{T}}, \frac{1}{b}}  \right)\\
  = &  
\sP\paren{\Ept{\barb{P}}{e^{\abs{\lambda S_T} - \lambda^2 \tilde{S}_T}} 
\ge   \frac{1}{\delta_j}\Ept{\set{D_l}_{l\le T}}{\Ept{\barb{P}}{e^{\abs{\lambda S_T} - \lambda^2 \tilde{S}_T}}}\Bigg| \lambda \in \bracket{\frac{1}{b}\sqrt{\frac{\ln \frac{2}{\delta}}{T}}, \frac{1}{b}}} \\
 \le&
\max_{\lambda \in \bracket{\frac{1}{b}\sqrt{\frac{\ln \frac{2}{\delta}}{T}}, \frac{1}{b}}}\sP\paren{\Ept{\barb{P}}{e^{\abs{\lambda S_T} - \lambda^2 \tilde{S}_T}} 
\ge   \frac{1}{\delta_j}\Ept{\set{D_l}_{l\le T}}{\Ept{\barb{P}}{e^{\abs{\lambda S_T} - \lambda^2 \tilde{S}_T}}}}\\
\labelrel\seq{eq275}& \max_{\lambda \in \set{\lambda_1, \dots, \lambda_J}}\sP\paren{\Ept{\barb{P}}{e^{\abs{\lambda S_T} - \lambda^2 \tilde{S}_T}} 
\ge   \frac{1}{\delta_j}\Ept{\set{D_l}_{l\le T}}{\Ept{\barb{P}}{e^{\abs{\lambda S_T} - \lambda^2 \tilde{S}_T}}}} \\ 
\le & \sum_{j=1}^J \sP\paren{\Ept{\barb{P}}{e^{\abs{\lambda_j S_T} - \lambda_j^2 \tilde{S}_T}} 
\ge   \frac{1}{\delta_j}\Ept{\set{D_l}_{l\le T}}{\Ept{\barb{P}}{e^{\abs{\lambda_j S_T} - \lambda_j^2 \tilde{S}_T}}}}\\
\labelrel\le{eq279} & J \times \delta_j  = \delta, 
\beq 
where \eqref{eq275} we use the discretization of $\lambda$, where \eqref{eq279} we apply Lemma~ \ref{mkv}. 

Thus, we get for all $\lambda \in \bracket{\frac{1}{b}\sqrt{\frac{\ln \frac{2}{\delta}}{T}}, \min\set{\sqrt{\frac{\sD_{KL}(P \| \barb{P}) + \ln \frac{2}{\delta}}{\Ept{P}{\tilde{S}_T}}},\frac{1}{b}}}$, we can obtain the inequality as in Equation~\eqref{eq214}  
\beq 
\label{eq285}
\abs{ \Ept{P}{S_T(\theta)}} &\le 2\sqrt{\Ept{P}{\tilde{S_T}}\paren{\sD_{KL}(P \| \barb{P}) + \ln \frac{2}{\delta_j}}} \\
\le& 2\sqrt{\Ept{P}{\tilde{S_T}}\paren{\sD_{KL}(P \| \barb{P}) + \ln \frac{2J}{\delta}}}\\
=& 2NH\sqrt{T \paren{\sD_{KL}(P \| \barb{P}) + \ln \frac{2J}{\delta}}} = 2NH\sqrt{T \paren{\sD_{KL}(P \| \barb{P}) + \ln \frac{\gO(\ln T)}{\delta}}},
\beq 
with probability at least $1-\delta$. 

If $\min\set{\sqrt{\frac{\sD_{KL}(P \| \barb{P}) + \ln \frac{2}{\delta}}{\Ept{P}{\tilde{S}_T}}},\frac{1}{b}} = \frac{1}{b}$, which implies 
\beq
\label{eq291}
\Ept{P}{\tilde{S}_T} \le b^2(\sD_{KL}(P \| \barb{P}) + \ln \frac{2}{\delta}) 
\beq , and for this value of $\lambda = \frac{1}{b}$,
we put Equation~\eqref{eq291} back to Equation~\eqref{eq208}, 
 
then we get,
\beq  
    \abs{ \Ept{P}{S_T(\theta)}} \le b\sD_{KL}(P \| \barb{P}) + b\ln \frac{2}{\delta}  + \frac{1}{b} \Ept{P}{\tilde{S}_T} \\
    \le  2b(\sD_{KL}(P \| \barb{P}) + \ln \frac{2}{\delta}) \le 2NH(\sD_{KL}(P \| \barb{P}) + \ln \frac{2}{\delta}). 
\beq

Then under the same argument we did previously for Equation~\eqref{eq285}, we have 
\beq 
    \abs{ \Ept{P}{S_T(\theta)}} 
    &\le  2b(\sD_{KL}(P \| \barb{P}) + \ln \frac{2}{\delta_j}) \\
    &\le 2NH(\sD_{KL}(P \| \barb{P}) + \ln \frac{2J}{\delta}) =  2NH\paren{\sD_{KL}(P \| \barb{P}) + \ln \frac{\gO(\ln T)}{\delta}} . 
\beq 
Note, the range of $\lambda$ depends on the $\tilde{S}_T$, which is sample dependent, thus we have the bound also depends on the sample.

Now if we let $\tilde{S}_T$ follow the definition in \eqref{azs}, now $\Ept{P}{\tilde{S}_T} = \tilde{S}_T$ since $\tilde{S}_T$ is not random. Now the $\lambda$ does not have an upper bound. We use the geometric sequence over the range of $\bracket{\sqrt{\frac{\ln \frac{2}{\delta}}{\tilde{S}_T}}, \infty}$, where $\sqrt{\frac{\ln \frac{2}{\delta}}{\tilde{S}_T}}$ is given when $\sD_{KL}(P \| \barb{P}) = 0$. We use the same argument, let $\lambda_j = c^j\sqrt{\frac{\ln \frac{2}{\delta}}{\tilde{S}_T}}$ for some $c > 1$ and $j\ge 0$. 
For given value of $\sD_{KL}(P \| \barb{P})$, the optimal $\lambda_j$ in \eqref{eq69} equals to $\sqrt{\frac{\sD_{KL}(P \| \barb{P}) + \ln \frac{2}{\delta}}{\tilde{S}_T}}$, which requires $j$ is the solution of $c^j\sqrt{\frac{\ln \frac{2}{\delta}}{\tilde{S}_T}} =  \sqrt{\frac{\sD_{KL}(P \| \barb{P}) + \ln \frac{2}{\delta}}{\tilde{S}_T}}$, and we floor the value of $j$ to the nearest integer, which is $\floor{\ln \paren{\frac{\mathbb{D}_{KL}\left(P \| \barb{P}\right)}{\ln(\frac{2}{\delta})}+1}/(2\ln c)} \le \paren{\ln\paren{\frac{\mathbb{D}_{KL}\left(P \| \barb{P}\right)}{\ln(\frac{2}{\delta})}}+1}/(2\ln c)$. 

As the same procedures in Equation~\eqref{eq252} we used for deriving Equation~\eqref{eq285}, 
we go back to the proof of Corollary \ref{ucad}, 
we let $\delta = \sum_{j=0}^{\infty} \delta_j = \sum_{j=0}^{\infty} 2^{-(j+1)}\delta$, with $\delta_j = 2^{-(j+1)}\delta$, $j\in \sN_{\ge 0}$.

Then with a similar argument in Equation~\eqref{eq252}, for any $\delta_j = 2^{-(j+1)}\delta$, we have 
we have 

\beq 
\label{eq343}
\sP&\paren{\Ept{P}{\abs{\lambda S_T} - \lambda^2 \tilde{S}_T } >= \sD_{KL}(P \| \barb{P}) + \ln \frac{2}{\delta_j} \Bigg| \lambda \in \bracket{\sqrt{\frac{\ln \frac{2}{\delta}}{\tilde{S}_T}}, \infty} }\\ 
\seq& \max_{\lambda \in \set{\lambda_1, \dots}}\sP\paren{\Ept{\barb{P}}{e^{\abs{\lambda S_T} - \lambda^2 \tilde{S}_T}} 
\ge   \frac{1}{\delta_j}\Ept{\set{D_l}_{l\le T}}{\Ept{\barb{P}}{e^{\abs{\lambda S_T} - \lambda^2 \tilde{S}_T}}}} \\ 
\le & \sum_{j=1}^{\infty} \sP\paren{\Ept{\barb{P}}{e^{\abs{\lambda_j S_T} - \lambda_j^2 \tilde{S}_T}} 
\ge   \frac{1}{\delta_j}\Ept{\set{D_l}_{l\le T}}{\Ept{\barb{P}}{e^{\abs{\lambda_j S_T} - \lambda_j^2 \tilde{S}_T}}}}\\
\le & \sum_{j=1}^{\infty} \times \delta_j  = \sum_{j=1}^{\infty} \delta 2^{-(j+1)} = \delta. 
\beq

In the end, we get for all $\lambda \in \bracket{\sqrt{\frac{\ln \frac{2}{\delta}}{\tilde{S}_T}}, \infty}$, we can obtain the inequality as in Equation~\eqref{eq214} 
\beq 
\label{eq40}
\abs{ \Ept{P}{S_T(\theta)}} &\le 2\sqrt{\tilde{S_T}\paren{\sD_{KL}(P \| \barb{P}) + \ln \frac{2}{\delta_j}}} \\
&\le 2\sqrt{\tilde{S_T}\paren{\sD_{KL}(P \| \barb{P}) + \ln \frac{2*2^j}{\delta}}}\\
&\le 2\sqrt{\sum_{l=1}^T\frac{(a_l-b_l)^2}{8} \paren{\sD_{KL}(P \| \barb{P}) + \ln \frac{2}{\delta} + \frac{\ln 2}{2\ln c}\paren{\ln\paren{\frac{ \mathbb{D}_{KL}\left(P \| \barb{P}\right)}{\ln(\frac{2}{\delta})}}+1}}}\\
&\le \sqrt{2}NH\sqrt{T \paren{\sD_{KL}(P \| \barb{P}) + \ln \frac{2}{\delta} + \frac{\ln 2}{2\ln c}\paren{\ln\paren{\frac{ \mathbb{D}_{KL}\left(P \| \barb{P}\right)}{\ln(\frac{2}{\delta})}}+1}}}.
\beq 

By utilizing Equation~\eqref{eq7}, the joint distribution can be decomposed into products of independent distributions that are solely dependent on the preceding episode, which can be successively absorbed into the filtration we defined. By  $P = \prod_{l=0}^{T-1}P_{l}$ and $\barb{P} = \prod_{l=0}^{T-1}\barb{P}_{l}$, we have 
\beq 
\label{eq371}
\mathbb{D}_{KL}\left(P \| \barb{P}\right) = \sum_{l=1}^T\mathbb{D}_{KL}\left(P_{l-1} \| \barb{P}_{l-1}\right), 
\beq 
putting Equation~\eqref{eq371} back into \eqref{eq40}, we have 
\beq 
\abs{ \Ept{P}{S_T(\theta)}} 
\le \sqrt{2}NH\sqrt{T \paren{\sum_{l=1}^T\mathbb{D}_{KL}\left(P_{l-1} \| \barb{P}_{l-1}\right) + \ln \frac{2}{\delta} + \frac{\ln 2}{2\ln c}\paren{\ln\paren{\frac{ \sum_{l=1}^T\mathbb{D}_{KL}\left(P_{l-1} \| \barb{P}_{l-1}\right)}{\ln(\frac{2}{\delta})}}+1}}}.
\beq 

 Next, by Lemma~\ref{bound_for_kl}, we have $$\sum_{l=1}^{T-1}\mathbb{D}_{KL}\left(P_{l} \| \barb{P}_{l}\right) \leq \frac{2\lambda^2 r^2}{s_{\min} (1-\alpha)^2} \frac{1 - \alpha^{2(T-1)}}{1-\alpha^2}.$$

Return to $T \coloneqq  K/N$, then take $\delta = 2\exp{-K}$, choose $c$ such that $\frac{\ln 2}{2\ln c} =1$, and by the inequality $\sqrt{a+b} \leq \sqrt{a} + \sqrt{b}$ and $\ln(K) + K \leq \sqrt{2} K$ for any $K>0$, some basic algebra, we 
get the final bound in Theorem~\ref{thm:main} Equation~\eqref{eq:pac_bound}.
\end{proof}

\begin{lemma}
\label{bound_for_kl}
   Suppose Assumption~\ref{assump:1} holds. Then, for any $l \in \{1, \dots, T-1\}$, the following bound on the KL divergence holds:
   $$
   \sum_{l=1}^{T-1} \mathbb{D}_{KL}(P_l \| \barb{P}_l) \leq \frac{2\lambda^2 r^2}{s_{\min} \, (1-\alpha)^2} \cdot \frac{1 - \alpha^{2(T-1)}}{1 - \alpha^2}.
   $$
\end{lemma}
\begin{proof}
For any $l$, we denote $l$th $\lambda$ as $\lambda_l$, so $\lambda_1 = \lambda$, $\lambda_l = \alpha^{l-1}\lambda$, where $\lambda$ and $\alpha$ are introduced in Algorithm~\ref{alg:mrl-0}. First by the property of KL divergence, we have 
$$
\mathbb{D}_{KL}\left(P_l \| \barb{P}_l\right) \leq \frac{2\|P_l - \barb{P}_l \|_\infty^2}{s_{min}}. 
$$ 
Further, given the updating rule, $\barb{P}_l = \lambda_l \barb{P}_{l-1} + (1-\lambda_l) P_l$, we have 
\begin{align*}
&\|P_l - \barb{P}_l \|_\infty = \|P_l - \lambda_l \barb{P}_{l-1} - (1-\lambda_l) P_l \|_\infty   = \|\lambda_l P_l-  \lambda_l \barb{P}_{l-1} \|_\infty \\ 
& = \|\lambda_l P_l - \lambda_l (\lambda_{l-1} \barb{P}_{l-2} + (1-\lambda_{l-1}) P_{l-1})  \|_\infty = \|\lambda_l (P_l - P_{l-1}) + \lambda_l \lambda_{l-1}(P_{l-1} - \barb{P_{l-2}})  \|_\infty \\
& = \|\lambda_l (P_l - P_{l-1}) + \lambda_l \lambda_{l-1}(P_{l-1} - (\lambda_{l-2} \barb{P}_{l-3} + (1-\lambda_{l-2}) P_{l-2} ))  \|_\infty \\
&= \|\lambda_l (P_l - P_{l-1}) + \lambda_l \lambda_{l-1}(P_{l-1} - P_{l-2}) +\lambda_l \lambda_{l-1} \lambda_{l-2}(P_{l-2} - \barb{P}_{l-3})\|_\infty   \\
&=\cdots \\
&= \|\lambda_l (P_l - P_{l-1}) + \lambda_l \lambda_{l-1}(P_{l-1} - P_{l-2})+ \cdots + \lambda_l \lambda_{l-1}\cdots \lambda_2(P_2 - P_1) + \lambda_l \lambda_{l-1} \lambda_{l-2}\cdots \lambda_1(P_1 - \barb{P}_0)\|_\infty \\ 
&= \| \alpha^{l-1} \lambda_1 (P_l - P_{l-1}) + \alpha^{l-1+l-2} \lambda_1 (P_{l-1} - P_{l-2}) + \cdots + \alpha^{l-1+l-2 + \cdots +1}\lambda_1 (P_1 - P_0)\|_\infty \\
& \leq \lambda_1 r \frac{\alpha^{l-1}}{1-\alpha},  
\end{align*}

where we use, $\barb{P}_0 = P_0$, $\lambda_l = \alpha \lambda_{l-1}$, and$\|P_l - P_{l-1} \|_\infty \leq r$, we have $\mathbb{D}_{KL}\left(P_l \| \barb{P}_l\right) \leq \frac{2\lambda^2 r^2 \alpha^{2l-2}}{s_{\min} (1-\alpha)^2}$, 
thus we have 

$\sum_{l=1}^{T-1}\mathbb{D}_{KL}\left(P_{l} \| \barb{P}_{l}\right) \leq \sum_{l=1}^{T-1} \frac{2\lambda^2 r^2 \alpha^{2l-2}}{s_{\min} (1-\alpha)^2} \leq \frac{2\lambda^2 r^2}{s_{\min} (1-\alpha)^2} \frac{1 - \alpha^{2(T-1)}}{1-\alpha^2}.$

\end{proof}

\subsection{Proof of Theorem~\ref{lemma:sample complexity}}
\vspace{5pt}
\begin{lemma}[Sample Complexity For Policy Gradient]
\label{lemma_sample_complexity_pg}
Consider the setting of \cref{thm:main}. 
Given a small $\epsilon > 0$, with proper choice of learning rate $\beta$, If the number of iterations $T$ satisfies $T = \widetilde{\mathcal O}(\epsilon^{-4})$. 
Denote $\bar J =  \frac{1}{K}\sum_{i=1}^K 
\mathbb{E}_{\left\{\theta_l \right\}_{l=0}^{T-1} \sim P} [\mathbb{E}_{\left\{\tau_l \right\}_{l=1}^{T} \sim \mathscr{T}} [ J_{\mathcal{M}_i}(\pi_{\theta}) ] ].$

Then $ J^* - \bar J \leq \mathcal O(\epsilon).$
\end{lemma}
\begin{proof}
Refer the Corollary C.1 in \cite{yuan2022general} for details. 
\end{proof}
We then begin to prove Theorem~\ref{lemma:sample complexity}.
\begin{proof}
By the proof of Theorem~\ref{thm:main} , we have $|\bar J - \tilde J| \le  \mathscr R(\mathbb{D}_{KL}(P \| \barb{P})) $, 
then the following conditions holds 
\begin{enumerate}
    \item let $\mathscr R(\mathbb{D}_{KL}(P \| \barb{P})) \leq \frac{\epsilon}{2}$, 
    \item and let
$J^* - \bar J \leq \mathcal O(\frac{\epsilon}{2})$,
\end{enumerate} where condition 1 holds by Theorem~\ref{thm:main}, and condition 2 holds by \cref{lemma_sample_complexity_pg}. 
By satisfying both conditions 1 and 2, we obtain, 
$ J^* - \tilde J \leq \mathcal O(\epsilon).$ The value of $K$ can then be determined to satisfy these conditions.
\end{proof}

\section{Auxiliary Theorems and Lemmas}
\begin{lemma}[Popoviciu's inequality on variances]
\label{varlemma}
For bounded random variable $x \in [a,b]$, then $\var{x} \le \frac{(b-a)^2}{4}$
\end{lemma}

\begin{lemma}[Markov's Inequality, Equation 2.1 in \cite{wainwright2019high}]
\label{mkv}
For any non-negative random variable $x$, it holds that
$\sP(x \ge t) \le \frac{\ept{x}}{t}$. Taking $t = \frac{\ept{x}}{\delta}$, where $\delta \in (0,1)$, it results in with probability at least $1-\delta$, $0\le x \le \frac{\ept{x}}{\delta}$. 
\end{lemma}

\begin{theorem}[Azuma-Hoeffding Inequality, Corollary 2.20 in \cite{wainwright2019high}]
\label{thm:ahi}
For a sequence of Martingale Difference Sequence random
variable $\set{D_l}_{l=1}^T$,
 if we have $D_l \in [a_l, b_l]$
almost sure for some constant $[a_l, b_l]$ and $l = 1, 2, \dots, T$, the summation $S_T := \sum_{l=1}^T D_l$, and let $\tilde{S}_T = \frac{\sum_{l=1}^T(b_l-a_l)^2}{8}$ then:
\begin{equation}
    \begin{aligned}
        \sP\paren{\abs{S_T} \ge t } \le 2e^{\frac{-t^2}{\tilde{S}_T}}
    \end{aligned}
\end{equation}
Equivalently, the moment-generating function satisfies
\beq
\label{azm}
    \ept{e^{\lambda S_T}} \le e^{\lambda^2\tilde{S}_T}.
\beq
Furthermore, if we choose $t = \sqrt{\frac{1}{2}\ln \frac{2}{\delta}\sum_{l=1}^T(b_l-a_l)^2}$, we get $\sP\paren{|S_T| > \sqrt{\frac{1}{2}\ln \frac{2}{\delta}\sum_{l=1}^T(b_l-a_l)^2}} \le \delta$.
\end{theorem}

\begin{theorem}[Freedman’s inequality, Theorem 1.6 in  \cite{freedman1975tail}]
\label{lm71}
    Let $\gF_T$ and  $\set{D_l}_{l\le T}$ follow the definition in Equation~\eqref{mtg}, and let $\abs{D_l} \le b$ with probability at least 1 and $\ept{D_l|\gF_{l-1}} = 0$. Let $S_T \coloneqq \sum_{l=1}^T D_l$ and $\tilde{S}_T \coloneqq \sum_{l=1}^T \ept{D_l^2|\gF_{l-1}}$. Then for any $\lambda \in [0, \frac{1}{b}]$
\beq 
\Ept{\set{D_l}_{l\le T}}{e^{\lambda S_T - \lambda^2 \tilde{S}_T}} \le 1
\beq 
\end{theorem}

\begin{proof}
\beq 
\label{eq78}
\Ept{D_T}{e^{\lambda D_T} | \gF_{T-1}} \labelrel{\le}{eq73} & \Ept{D_T}{1 + \lambda D_T + \lambda^2 D_T^2| \gF_{T-1}}\\
= & 1 + \lambda^2 \Ept{D_T}{D_T^2 | \gF_{T-1} } \\
\labelrel{\le}{eq75} & e^{\lambda^2 \Ept{D_T}{D_T^2 | \gF_{T-1}}}
\beq 
Where \eqref{eq73} holds since $e^x \le 1 + x + x^2$ for $0 < x \le 1$, thus, we require $\lambda D_T \le 1$, so $\lambda \le \frac{1}{b}$, and \eqref{eq75} holds by $1+x \le e^x$. 
Now we have

\beq 
\Ept{\set{D_l}_{l\le T}}{e^{\lambda S_T - \lambda^2 \tilde{S}_T}} &\labelrel={eq347} 
\Ept{\set{D_l}_{l\le T}}{e^{\lambda S_{T-1} - \lambda^2 \tilde{S}_{T-1} + \lambda D_T - \lambda^2 \ept{(D_T)^2|\gF_{T-1}}}}\\
& \labelrel={eq349} \Ept{\set{D_l}_{l\le T-1}}{e^{\lambda S_{T-1} - \lambda^2 \tilde{S}_{T-1}} \times \Ept{D_T}{e^{\lambda D_T}|\gF_{T-1}} \times e^{-\lambda^2\ept{(D_T)^2|\gF_{T-1}}}} \\
&\labelrel{\le}{eq90} \Ept{\set{D_l}_{l\le T-1}}{e^{\lambda S_{T-1} - \lambda^2 \tilde{S}_{T-1}}}\\
&\le ... \\ &\le 1.
\beq 

where \eqref{eq347} holds by the definition of $\tilde{S}_T$,  and \eqref{eq349} holds by the definition of the definition $D_T|\gF_{T-1}$,   
where \eqref{eq90} holds by \eqref{eq78}, and in the last step above we have recursively applied the above argument.  
\end{proof}

\begin{theorem}[Donsker$-$Varadhan’s Representation formula, \cite{donsker1983asymptotic}] 
\label{thmdvv}
Given a probability space $(\gX, \gB)$ and a bounded real-valued function $f$, where $f(x)$ is a measurable function $f:\gX \rightarrow \R$, $x$ is a random variable, and any two probability distributions $P_0$ and $P$ over $\gX$ (or, if $\gX$ is uncountably infinite, two probability density functions),   
\begin{equation}
\begin{aligned}
\mathbb{D}_{KL}\left(P \| \barb{P}\right)     \ge \Ept{P}{f(x)} - \ln\Ept{\barb{P}}{e^{f(x)}}.
\end{aligned}
\end{equation}

\end{theorem}
The $\ln\Ept{\barb{P}}{e^{f(x)}}$ on the right-hand side is the cumulant generating function. 
These lemmas have been commonly used in the theory of online learning ~\citep{pang2021reinforcement,kang2022efficient,yuan2022general,liu2021regret,kang2023robust}.
\section{Policy Function Parameter $\theta$ With A Gaussian Prior}

\subsection{Neural Network Parametrization}
\label{a300}
In Equation~\eqref{eq20}, the parameter $\theta$ can represent the weights of a neural network. Here, we provide details on how we set up the parameter updates for the neural network weights. Let $\theta = (w_r, b_r)$ denote the random weights and biases of the $r$-th ($r \in \sN_{\ge 1}$) network layer. Additionally, let $\epsilon_r$ and $\epsilon_{b_r}$ be multivariate standard normal distributed random variables. The random weights $w_r$ and biases $b_r$ are defined as follows:
\beq 
\label{eq54}
w_r = \mu_{r} \odot (1 + \gamma_r \epsilon_r),  \gamma_r = \ln(1+\exp{\delta_r}),
\beq 

\beq 
\label{eq55}
b_r = \mu_{b_r} \odot (1 + \gamma_{b_r} \epsilon_{b_r}),  \gamma_{b_r} = \ln(1+\exp{\delta_{b_r}}).
\beq 

This implies that $w_r$ and $b_r$ are multivariate normal distributed according to:

\beq 
w_r \sim \gN(\mu_{r}, \gamma_r^2 \diag{\mu_{r}^2})), \quad b_r \sim \gN(\mu_{b_r}, \gamma_{b_r}^2 \diag{\mu_{b_r}^2}).
\beq 

During optimization in each iteration, a sample of $w_r$ and $b_r$ is drawn from the random network parameters to perform gradient descent.

The indirect sampling according to Equations \eqref{eq54} and \eqref{eq55} ensures that the parameters $\mu_{r}, \gamma_r, \mu_{b_r}, \gamma_{b_r}$ can be updated. The normal prior $\barb{p}(\theta)$ is defined as:

\beq 
w_{\barb{r}} \sim \gN(\barb{\mu}_r, \barb{\gamma}_r^2 ), \quad b_{\barb{r}} \sim \gN(\barb{\mu}_{b_r}, \barb{\gamma}_{b_r}^2).
\beq 

Thus, the posterior distribution for the neural network is given as $p(\theta|D) = \barb{p}(\theta)p(D|\theta)/\int \barb{p}(\theta)p(D|\theta)d\theta$, where $p(D|\theta) \coloneqq p(g(\tau)|\theta)$ is the data likelihood. The exact likelihood function $p(D|\theta)$ and posterior policy $p(\theta|D)$ are left as future research, as mentioned in Section~\ref{post-prior-dist} "Posterior Distribution and Prior Distribution".

Instead of analytically deriving $p(\theta|D)$, we assume it belongs to a common distribution family of the prior, but with unknown parameters, which are updated by minimizing the upper bound. Therefore, we approximate the posterior $p(\theta|D)$ by a proposed distribution $q(\theta)$.

Following this approach, we can approximate the posterior $p(\theta|D)$ by updating the parameters of $q(\theta)$ using the indirect sampling chain rule. We first sample a $\theta \sim (\gN_{w_r}, \gN_{b_r})$, then evaluate the right-hand side in Equation~\eqref{eq9}. 

We can calculate the derivatives of $\hat{U}(P, \set{\gM_i}_{i\in[N]}, \set{\theta_{l-1,j}}_{j\in[M]}; \mu, \sigma,  \barb{\mu}, \barb{\sigma})$ with respect to $\mu_{r}, \mu_{b_r}$, $\delta_r, \delta_{b_r}$ and $\barb{\mu}_{r}, \barb{\mu}_{b_r}$, $\barb{\delta}_r, \barb{\delta}_{b_r}$ in Equations \eqref{eq54} and \eqref{eq55}, which we used in our implementation.

  \section{Environment and Experiment}
\subsection{Supplementary Experiment}
We conducted additional experiments to examine the influence of different \(\lambda\) values on the final rewards for selective environments as shown in Table~\ref{table_lambda}.

\begin{table*}[h]
\centering
\caption{Final Rewards for Different \(\lambda\) Values}
\footnotesize
\renewcommand{\arraystretch}{1.1}
\setlength{\tabcolsep}{6pt} 

\begin{tabularx}{\textwidth}{l *{6}{>{\centering\arraybackslash}X}}
\toprule
\cellcolor{paleaqua}\textbf{Experiments} & \(\lambda = 0.84\) & \(\lambda = 0.86\) & \(\lambda = 0.88\) & \(\lambda = 0.90\) & \(\lambda = 0.92\) & \(\lambda = 0.94\)  \\ 
\midrule
\cellcolor{paleaqua}HalfCheetah (bodyparts) & 162 (7.9)   & 176 (6.6)   & 181 (8.1)   & \textbf{192 (10.2)}  & 184 (9.8)   & 174 (9.5)  \\ 
\cellcolor{paleaqua}Hopper (bodyparts)      & 302 (18.1)  & 318 (14.2)  & 321 (20.9)  & \textbf{345 (14.5)}  & 328 (18.1)  & 312 (15.9) \\ 
\cellcolor{paleaqua}Walker (gravity)        & 305 (18.8)  & 321 (23.6)  & \textbf{357 (21.7)}  & 341 (24.4)  & 329 (26.3)  & 314 (21.5) \\ 
\cellcolor{paleaqua}Ant (Forward-Backward)  & -4.48 (0.4) & -4.17 (0.4) & -4.10 (0.3) & \textbf{-3.59 (0.3)} & -3.97 (0.4) & -4.28 (0.3) \\ 
\cellcolor{paleaqua}Swimmer (Uniform)       & 14 (1.9)    & 16 (1.7)    & \textbf{18 (1.1)}    & 17 (1.3)    & 16 (1.8)    & 15 (1.6)   \\ 
\cellcolor{paleaqua}Humanoid (Direction)    & 350 (16.8)  & 373 (12.4)  & \textbf{388 (17.1)}  & 386 (11.7)  & 377 (16.8)  & 362 (15.4) \\ 
\bottomrule
\end{tabularx}

\label{table_lambda}
\end{table*}


%

The results show that $\lambda$ controls the trade-off between exploration and exploitation. Larger $\lambda$ emphasizes exploration by increasing the KL penalty, while smaller $\lambda$ prioritizes exploitation but risks overfitting.   Note that $\lambda=0.9$ consistently produces good result, so it can be set as the default and fine-tuned around this value if needed.

\begin{table*}[ht!]
    \centering
    \caption{Different Lifelong Environments}
    \resizebox{\textwidth}{!}{%
    \begin{tabular}{c c p{8cm}<{\centering} c}
        \toprule
        \multirow{3}{*}{\cellcolor{paleaqua}CartPole-GMM} & cart mass & $0.15[\mathcal{N}(1,0.1^2)+0.15\mathcal{N}(5,0.1^2)] +0.18[\mathcal{N}(2,0.1^2)+0.18\mathcal{N}(4,0.1^2)] +0.34\mathcal{N}(3,0.1^2)$  \\
       \cellcolor{paleaqua} & pole mass &  $0.15[\mathcal{N}(0.4,0.01^2)+\mathcal{N}(0.5,0.01^2)] +0.18[\mathcal{N}(0.2,0.01^2)+\mathcal{N}(0.3,0.01^2)] +0.34\mathcal{N}(0.1,0.01^2)$     \\
       \cellcolor{paleaqua} & pole length & $0.15[\mathcal{N}(0.3,0.01^2)+\mathcal{N}(0.7,0.01^2)] +0.18[\mathcal{N}(0.4,0.01^2)+\mathcal{N}(0.6,0.01^2)] +0.34\mathcal{N}(0.5,0.01^2)$    \\
        \midrule 
       \cellcolor{paleaqua} \multirow{1}{*}{CartPole-Uniform} & cart mass &$\mathcal{U}(1,5) $ \\
     \cellcolor{paleaqua}  & pole mass &$\mathcal{U}(0.1,0.5)$ \\
     \cellcolor{paleaqua}  & pole length & $\mathcal{U}(0.3,0.7)$
        \\
        \midrule  
        \multirow{3}{*}{\cellcolor{paleaqua}LunarLander-GMM} & main engine power & $0.15[\mathcal{N}(11,0.1^2)+0.18\mathcal{N}(12,0.1^2)] +0.34[\mathcal{N}(13,0.1^2)+0.18\mathcal{N}(14,0.1^2)] +0.15\mathcal{N}(15,0.1^2)$   \\
       \cellcolor{paleaqua} & side engine power &  $0.15[\mathcal{N}(0.45,0.01^2)+0.18\mathcal{N}(0.55,0.01^2)] +0.34[\mathcal{N}(0.65,0.01^2)+ 0.18\mathcal{N}(0.75,0.01^2)] +0.15\mathcal{N}(0.85,0.01^2)$      
        \\
        \midrule 
        \cellcolor{paleaqua}  \multirow{1}{*}{LunarLander-Uniform} & main engine power & $\mathcal{U}(3,20) $ \\
        \cellcolor{paleaqua}& side engine power &  $\mathcal{U}(0.15,0.95)$ 
        \\
        \midrule
      \cellcolor{paleaqua} Swimmer-Uniform & movement direction &$\theta \sim \mathcal{U}(0, \pi)$ \\
      \midrule      
\cellcolor{paleaqua} Humanoid-Direction-Uniform    & movement direction &$\theta \sim \mathcal{U}(0, 2\pi)$  \\
        \midrule 
        \cellcolor{paleaqua} Ant-Direction-Uniform & goal direction & $\theta \sim \mathcal{U}(0, 2\pi)$ \\ 
        \midrule 
        \cellcolor{paleaqua} Ant-Forward-Backward-Bernoulli    & movement direction &$\theta \sim \text{Categorical}(0, \pi; 0.5)$    \\
        \midrule 
\cellcolor{paleaqua}  
        & HalfCheetah-gravity & Gravity sampled from $\mathcal{U}(0.5g,1.5g)$ \\
        \cellcolor{paleaqua} & HalfCheetah-bodyparts & Mass and size scaling of torso, thigh, leg $\sim \mathcal{U}(0.5,1.5)$ \\
\cellcolor{paleaqua}\multirow{1}{*}{Physics-based Variations}  
        & Hopper-gravity & Gravity sampled from $\mathcal{U}(0.5g,1.5g)$ \\
       \cellcolor{paleaqua}\multirow{1}{*}{See \cite{mendez2020lifelong, fu2022model}} & Hopper-bodyparts & Mass and size scaling of body parts $\sim \mathcal{U}(0.5,1.5)$ \\
       \cellcolor{paleaqua} & Walker-gravity & Gravity sampled from $\mathcal{U}(0.5g,1.5g)$ \\
       \cellcolor{paleaqua} & Walker-bodyparts & Mass and size scaling of body parts $\sim \mathcal{U}(0.5,1.5)$ \\
    \bottomrule
    \end{tabular}}    
    \label{tab:dist_dynamic_cartpole}
\caption*{\footnotesize Note: \( g \) represents the standard gravity acceleration (\( 9.81 \text{ m/s}^2 \)).}
\end{table*}

\subsection{OpenAI and MAMuJoCo Environment}
\label{openai mamujoco}

\begin{figure}[htb]
    \centering
    \begin{subfigure}{0.18\textwidth}
        \includegraphics[width=\textwidth, height=20mm]{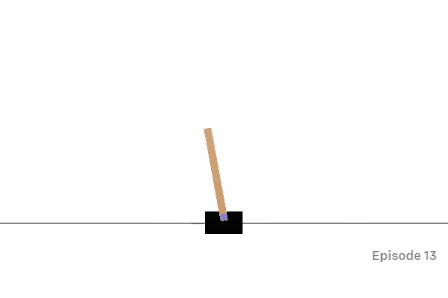}
        \caption{CartPole}
        \label{subfig:cart_env}
    \end{subfigure}\hfill
    \begin{subfigure}{0.18\textwidth}
        \includegraphics[width=\textwidth, height=20mm]{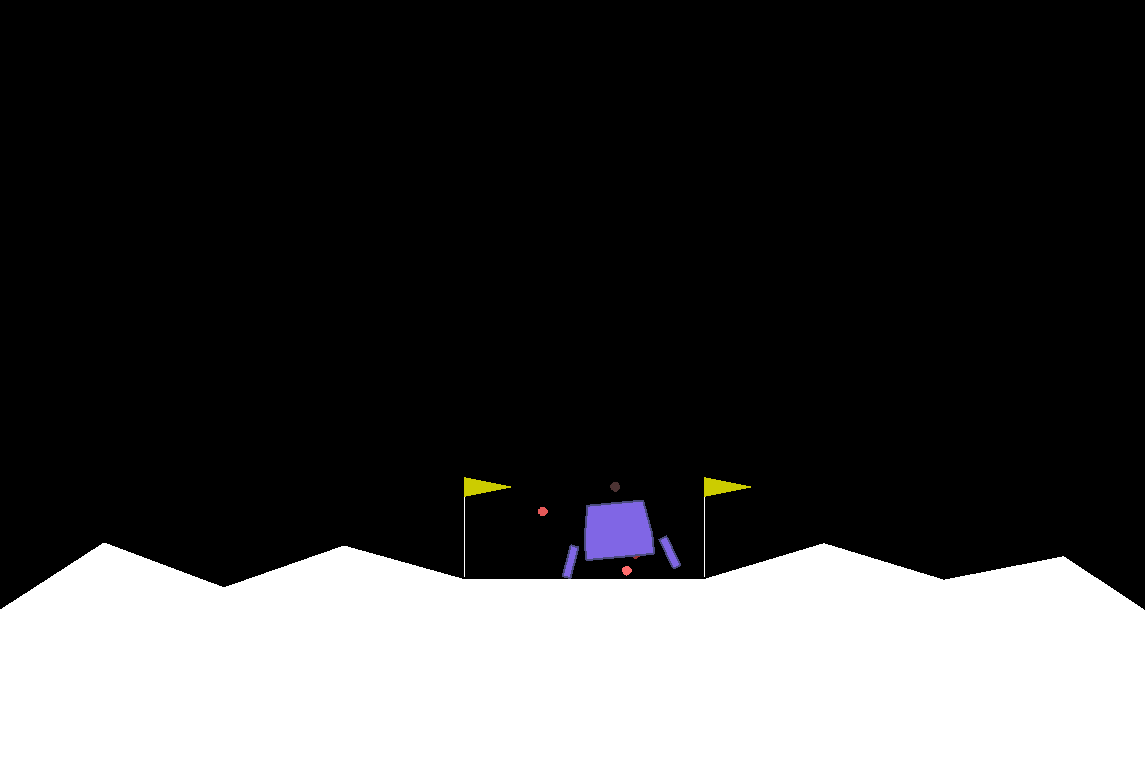}
        \caption{LunarLander}
        \label{subfig:lunar_env}
    \end{subfigure}\hfill
    \begin{subfigure}{0.18\textwidth}
        \includegraphics[width=\textwidth, height=20mm]{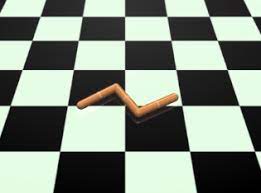}
        \caption{Swimmer}
        \label{subfig:swimmer_env}
    \end{subfigure}\hfill
    \begin{subfigure}{0.18\textwidth}
        \includegraphics[width=\textwidth, height=20mm]{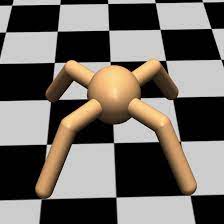}
        \caption{Ant}
        \label{subfig:ant_env}
    \end{subfigure} \\
    \begin{subfigure}{0.18\textwidth}
        \includegraphics[width=\textwidth, height=20mm]{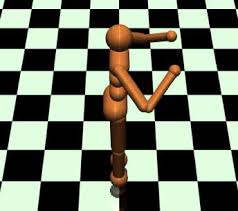}
        \caption{Humanoid}
        \label{subfig:human_env}
    \end{subfigure}\hfill
    \begin{subfigure}{0.18\textwidth}
        \includegraphics[width=\textwidth, height=20mm]{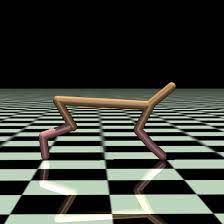}
        \caption{Cheetah}
        \label{subfig:cheetah_env}
    \end{subfigure}\hfill
    \begin{subfigure}{0.18\textwidth}
        \includegraphics[width=\textwidth, height=20mm]{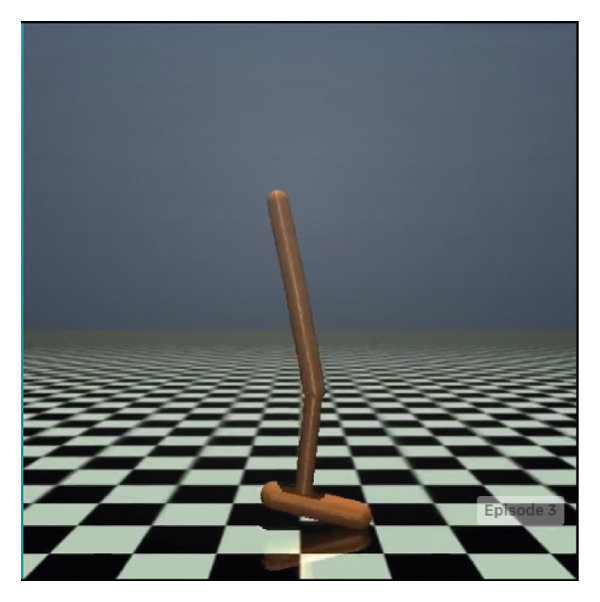}
        \caption{Hopper}
        \label{subfig:hopper_env}
    \end{subfigure}\hfill
    \begin{subfigure}{0.18\textwidth}
        \includegraphics[width=\textwidth, height=20mm]{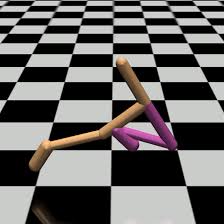}
        \caption{Walker}
        \label{subfig:walker_env}
    \end{subfigure}

    \caption{The illustration of environments. (a) CartPole, (b) LunarLander, (c) Swimmer, (d) Ant, (e) Humanoid, (f) Cheetah, (g) Hopper, (h) Walker}
    \label{fig:envs}
\end{figure}

We evaluate our method on various OpenAI Gym and MuJoCo-based lifelong RL environments (see Figure~\ref{fig:envs}), introducing structured variations in key dynamics to encourage continual adaptation. Table~\ref{tab:dist_dynamic_cartpole} summarizes these modifications.

\textbf{CartPole} and \textbf{LunarLander} modify mass, engine power, and length using Gaussian mixtures and uniform distributions. \textbf{Swimmer} and \textbf{Ant} introduce random movement directions, requiring adaptive control for diverse locomotion tasks. \textbf{HalfCheetah, Hopper, and Walker} adjust gravity and body morphology, simulating real-world physical variations that impact stability and efficiency.

\subsection{Hyper-Parameters}
\label{hyper}

Table~\ref{tab:epic_hyperpara} 
list hyperparameters used in EPICG. Among these hyperparameters, the frequency of lifelong update, i.e., $N$ is very important and closely related to the performance of both algorithm. Therefore, $N$ is choosen carefuly for each environment, whose values are shown in Table~\ref{tab:N_lifelong}. For hyperparameters of other methods, we use the original source code with parameters and model architectures
suggested in the original paper. We believe hyperparameters can further be tuned online and it will be our future work~\citep{ding2022syndicated,kang2024online}. The experiments were done in the GeForce RTX 2080 Ti GPU with 10 GB Memories.

\begin{table}[H]
    \centering
    \begin{tabular}{c|p{10cm}<{\centering}}
    \hline
    Hyperparameters & Values \\
    \hline
         taks ($K$) & 2000 or 1000 \\
         learning rate & $10^{-4}$ 
         \\$\beta$ & $10^{-4}$ 
         \\
         $N$ & chosen the best from $\{5, 10, 25, 50\}$\\
         initial value of $\lambda$ & 0.9 \\
         decay factor of $\lambda$ & 0.95 \\
     \hline
         
    \end{tabular}
    \caption{Hyparameters of EPICG}
    \label{tab:epic_hyperpara}
\end{table}

\begin{table}[H]
    \centering
    \begin{tabular}{cc}
    \hline
    Environments  & EPICG \\
    \hline
         Cartpole-GMM & \textcolor{black}{25}\\
         Cartpole-Uniform & \textcolor{black}{25}\\
         LunarLander-GMM & \textcolor{black}{25}\\
         LunarLander-Uniform & \textcolor{black}{25}\\
         Ant-Direction-Uniform & 
         \textcolor{black}{25}\\
         Ant-Forward-Backward-Bernoulli & \textcolor{black}{10}\\
         Swimmer-Uniform & 
          \textcolor{black}{25}\\
         Humanoid-Direction-Uniform & \textcolor{black}{10}\\
         HalfCheetah-gravity & \textcolor{black}{10}\\
         HalfCheetah-bodyparts & \textcolor{black}{10}\\
         Hopper-gravity & \textcolor{black}{25}\\
         Hopper-bodyparts & \textcolor{black}{25}\\
         Walker-gravity & \textcolor{black}{25}\\
         Walker-bodyparts & \textcolor{black}{25}\\
     \hline
         
    \end{tabular}
    \caption{Lifelong update frequency of EPICG}
    \label{tab:N_lifelong}
\end{table}

\end{document}